\documentclass[sts, twocolumn]{imsart}

\usepackage{algorithmicx,algorithm}
\usepackage{algcompatible}
\usepackage{amsthm}
\usepackage{amsmath}
\usepackage{natbib}
\usepackage[colorlinks,citecolor=blue,urlcolor=blue,filecolor=blue,backref=page]{hyperref}
\usepackage{graphicx}

\usepackage{algpseudocode}
\usepackage{amssymb}
\usepackage{bbm}
\usepackage{booktabs}
\usepackage{hhline}
\usepackage{mdframed}
\usepackage{multirow}
\usepackage{subcaption}
\usepackage{xcolor}

\startlocaldefs
\numberwithin{equation}{section}
\theoremstyle{plain}
\newtheorem{prop}{Remark}

\definecolor{lightpurple1}{RGB}{240, 230, 245} 
\definecolor{purple1}{RGB}{102, 0, 153}

\newmdtheoremenv[outerlinewidth=2,
leftmargin=40,
rightmargin=40,
backgroundcolor=lightpurple1,
outerlinecolor=purple1,
innertopmargin=0pt,
splittopskip=\topskip,
skipbelow=\baselineskip,
skipabove=\baselineskip,
ntheorem]{obs}%
{Observation}

\endlocaldefs

\DeclareMathOperator*{\argmax}{\arg\max}

\newcommand{\tp}{\textcolor{black}}

\begin{document}

\begin{frontmatter}


\title{Approximate blocked Gibbs sampling for Bayesian neural networks}
\runtitle{Approximate blocked Gibbs sampling for Bayesian neural networks}



\author{\fnms{Theodore} \snm{Papamarkou}\corref{}\ead[label=e1]{
t.papamarkou@manchester.ac.uk}}
\address{Department of Mathematics,
The University of Manchester, Manchester, UK}

\runauthor{T. Papamarkou}

\begin{abstract}
In this work, minibatch MCMC sampling
for feedforward neural networks is made more feasible.
To this end,
it is proposed to sample subgroups of parameters
via a blocked Gibbs sampling scheme.
By partitioning the parameter space,
sampling is possible irrespective of layer width.
It is also possible to alleviate vanishing acceptance rates
for increasing depth
by reducing the proposal variance in deeper layers.
Increasing the length of a non-convergent chain
increases the predictive accuracy in classification tasks,
so avoiding vanishing acceptance rates and
consequently enabling longer chain runs
have practical benefits.
Moreover, non-convergent chain realizations
aid in the quantification of predictive uncertainty.
An open problem is how to perform minibatch MCMC sampling
for feedforward neural networks
in the presence of augmented data.
\end{abstract}




\end{frontmatter}

\section{Introduction}

\textbf{Scope.}
This paper renders
feedforward neural networks
more amenable to approximate MCMC sampling
of their parameters
by splitting the parameters into subgroups.
Moreover, it identifies several advantages
of such a sampling approach.

\textbf{Motivation.}
Why consider approximate MCMC 
sampling algorithms for
deep learning?
The answer stems from a general merit of MCMC,
namely uncertainty quantification.
This work demonstrates
how approximate MCMC sampling 
of neural network parameters
quantifies predictive uncertainty 
in classification problems.

\textbf{Limitations.}
Several impediments have inhibited the
adoption of MCMC in deep learning;
to name three notorious problems,
low acceptance rate,
high computational cost and
lack of convergence
typically occur.
See~\cite{papamarkou2022} for a relevant review.

\textbf{Potential.}
Empirical evidence herein suggests
a less dismissive view of
approximate MCMC in deep learning.
Firstly,
a sampling mechanism that takes into account the
neural network structure and that
partitions the parameter space into
smaller parameter blocks
retains higher acceptance rate.
Secondly,
minibatch MCMC sampling of neural network parameters
mitigates the computational bottleneck
induced by big data.
Bayesian marginalization,
which is used for making predictions and
for assessing predictive performance,
is also computationally expensive.
However, Bayesian marginalization
is embarrassingly parallelizable
across test points and along Markov chain length.
Thirdly,
if assessment of predictive uncertainty
via neural networks
is the intended outcome,
then MCMC convergence in parameter space
is viewed as a stepping stone
rather than as a pre-requirement for such an outcome.
A non-convergent Markov chain
acquires valuable predictive information.
\tp{In fact, it has been shown that
the posterior predictive density
in Bayesian neural networks can be restricted
to a symmetry-free subset of the parameter space~\citep{wiese2023}.}

\textbf{Contributions.}
The main contribution of this paper is
\tp{to propose minibatch blocked Gibbs sampling
for feedforward neural networks and
and to experimentally corroborate the
feasibility of such a sampling approach}.
Without optimizing prior specification,
vanishing acceptance rates are overcome
by partitioning the parameter space
into small blocks. Several observations
are drawn from an experimental study
of the proposed sampling scheme
for feedforward neural networks.
Firstly, it is observed that
partitioning the parameter space
allows to sample from it
under increasing width.
Secondly, such partitioning alleviates
vanishing acceptance rates in deeper layers
by reducing the proposal variance
as depth increases.
Thirdly, it is pointed out that
increasing the batch size
increases the predictive accuracy as expected,
as long as the batch size does not become large
to the point of yielding vanishing acceptance rates.
Fourthly, it is demonstrated that
letting the realization of a non-convergent chain
run longer increases the predictive accuracy.
Fifthly, it is confirmed that one
of the open problems is sampling
in the presence of augmented data.
Finally, it is demonstrated that
non-convergent chain realizations
aid in the quantification of predictive uncertainty.

\textbf{Paper structure.}
The paper is structured as follows.
Section~\ref{background}
reviews the MCMC literature
for deep learning.
Section~\ref{preliminaries}
revises some basic knowledge,
including the
Bayesian multilayer perceptron (MLP) model
and blocked Gibbs sampling.
Section~\ref{methodology}
introduces a finer node-blocked Gibbs (FNBG) algorithm
to sample MLP parameters.
Section~\ref{experiments}
utilizes FNBG sampling to
fit MLPs to three training datasets,
making predictions on three associated test datasets.
In Section~\ref{experiments},
numerous observations are made about the scope
of approximate MCMC in MLPs.
Section~\ref{discussion}
concludes the paper with a discussion
about future research directions
and about associated limitations.

\section{Literature review}
\label{background}

This section reviews
the literature on MCMC for neural networks.
Several other reviews of the topic exist,
see for instance
~\cite{titterington2004, wenzel2020, izmailov2021, papamarkou2022}.
New MCMC developments for neural networks,
which have appeared after the aforementioned reviews,
are included herein.

Four research directions have been mainly taken
to develop MCMC algorithms for neural networks.
Initially,
sequential Monte Carlo (SMC)
and reversible jump MCMC were applied
on feedforward neural networks.
At a second wave of development,
minibatch MCMC algorithms
became a mainstream approach.
More recently,
the focus has shifted to
Gibbs sampling algorithms and to
the construction of priors for Bayesian neural networks.

\subsection{SMC \& reversible jump MCMC}

In early stages of MCMC developments for neural networks,
SMC and reversible jump MCMC were applied
on MLPs and radial basis function networks
~\citep{andrieu1999, freitas1999, andrieu2000, freitas2001}.
For a historical context
of Bayesian approaches to neural networks, see
~\cite{titterington2004, papamarkou2022}.

\subsection{Minibatch MCMC}
\label{subsec:minibatch_mcmc}

In minibatch MCMC,
a target density is evaluated
on a subset (minibatch) of the data,
thus avoiding the computational cost
of MCMC iterations based on the entire data.
A stochastic gradient MCMC (SG-MCMC) algorithm
is a minibatch MCMC algorithm that uses
the gradient of the target density.
~\cite{welling2011} have employed the notion of minibatch
to develop a stochastic gradient Langevin dynamics (SG-LD)
Monte Carlo algorithm,
which is the first instance of SG-MCMC.
~\cite{chen2014} have introduced
stochastic gradient Hamiltonian Monte Carlo (SG-HMC),
which is another instance of SC-MCMC,
and applied it to infer the parameters of a Bayesian neural network
fitted to the MNIST dataset~\citep{lecun1998}.

SG-LD and SG-HMC are two SG-MCMC algorithms
that initiated approximate MCMC research
in machine learning.
Several variants of SG-MCMC have appeared ever since.
\cite{gong2018} have proposed an SG-MCMC scheme
that generalizes Hamiltonian dynamics with state-dependent drift and diffusion,
and have demonstrated the performance of this scheme
on convolutional and on recurrent neural networks.
\cite{zhang2020} have proposed cyclical SG-MCMC,
a tempered version of SG-LD with a
cyclical stepsize schedule.
Moreover,~\cite{zhang2020}
have showcased the performance of
cyclical SG-MCMC on a ResNet-18~\citep{he2016}
fitted to the CIFAR-10 and CIFAR-100 datasets~\citep{krizhevsky2009}.
\cite{alexos2022} have introduced structured SG-MCMC,
a combination of SG-MCMC and
structured variational inference~\citep{saul1995}.
Structured SG-MCMC employs SG-LD or SG-HMC to 
sample from a
factorized variational parameter posterior density.
\cite{alexos2022} have tested the performance
of structured SG-MCMC on ResNet-20~\citep{he2016} architectures
fitted to the
CIFAR-10,
SVHN~\citep{netzer2011} and
fashion MNIST~\citep{xiao2017} datasets.

\subsection{Gibbs sampling}
\label{subsec:gibbs_sampling}

Various Gibbs sampling algorithms have been developed
recently with large-scale inference in mind.
~\cite{bouchard2017} have introduced 
the particle Gibbs split-merge sampler
and have explored its performance
on four high dimensional datasets.
Split Gibbs samplers based on
the alternating direction method of multipliers optimization algorithm
have been developed to perform Bayesian inference
on large datasets and potentially on high-dimensional models
~\citep{vono2019,vono2022}.
Despite not having been applied so far to neural networks,
such particle Gibbs and split Gibbs samplers demonstrate that
the idea of splitting parameters or auxiliary variables into subgroups
provides one way of attacking the problem of large-scale inference.

\cite{grathwohl2021} have introduced
the Gibbs-with-gradients (GWG) sampler,
a general and scalable approximate sampling strategy
for probabilistic models with discrete variables.
GWG is related to the adaptive Gibbs sampler~\citep{latuszynski2013}.
~\cite{grathwohl2021} have trained GWG
on restricted Boltzmann machines,
which are generative stochastic neural networks,
and have compared GWG to blocked Gibbs sampling,
using samples from the latter as the ground truth.

Minibatch MCMC (Subsection~\ref{subsec:minibatch_mcmc}) and
Gibbs samplers (current Subsection~\ref{subsec:gibbs_sampling})
do not constitute two mutually exclusive classes of algorithms.
To elaborate on the involved ontology
of minibatch MCMC and Gibbs samplers,
three remarks are made.
Firstly, HMC can be formulated
as a Gibbs sampler~\citep{girolami2011}.
Secondly, each parameter subgroup
in blocked Gibbs sampling can be updated
via an MCMC sampling step.
For instance, if each parameter subgroup
is updated via a Metropolis-Hastings (MH),
Langevin dynamics (LD) or HMC sampling step,
then the corresponding sampler is known as
MH-within-Gibbs, LD-within-Gibbs or HMC-within-Gibbs.
Thirdly, the terminology SG-LD and SG-HMC
is used in software documentation
to refer to algorithms that sample
all neural network parameters at one sweep or layer-wise.
Nevertheless, when parameter sampling is conducted layer-wise,
SG-LD and SG-HMC are misnomers,
and the correct sampler names are
SG-LD-within-Gibbs and SG-HMC-within-Gibbs, respectively.

\subsection{Prior specification}

Prior specification for neural networks
was considered on the eve of the twenty-first century,
see~\cite{papamarkou2022} for a relevant review.
Research on prior specification for neural networks
has resurged recently,
as ridgelet priors~\citep{matsubara2021}
and functional priors~\citep{tran2022}
have been introduced.
The functional priors proposed by~\cite{tran2022}
have been designed for performing
approximate MCMC sampling in Bayesian deep learning.

\section{Preliminaries}
\label{preliminaries}


This section revises two topics,
the Bayesian MLP model for supervised classification
(Subsection~\ref{mlp_model})
and blocked Gibbs sampling
(Subsection~\ref{bgs}).
For the Bayesian MLP model,
the parameter posterior density
and posterior predictive probability mass function (pmf)
are stated.
Blocked Gibbs sampling provides a starting point
in developing the algorithm of Section~\ref{methodology}
for sampling from the MLP parameter posterior density.

\subsection{The Bayesian MLP model}
\label{mlp_model}


An MLP is a feedforward neural network comprising
an input layer, one or more hidden layers and an output layer
~\citep{rosenblatt1958,minsky1988,hastie2016}.
For a fixed natural number $\rho \ge 2$,
an index $j\in\{0,1,\dots,\rho\}$ indicates the layer.
In particular,
$j=0$ refers to the input layer,
$j\in\{1,2,\dots,\rho-1\}$ to one of the $\rho-1$ hidden layers,
and $j=\rho$ to the output layer.
Let $\kappa_{j}$ be the number of nodes in layer $j$,
and let
$\kappa_{0:\rho} = (\kappa_{0},\kappa_{1},\dots,\kappa_{\rho})$
be the sequence of node counts per layer.
$\mbox{MLP}(\kappa_{0:\rho})$
denotes an MLP with $\rho-1$ hidden layers and
$\kappa_j$ nodes at layer $j$.

An $\mbox{MLP}(\kappa_{0:\rho})$
with $\rho-1$ \textit{hidden layers} and
$\kappa_j$ \textit{nodes} at layer $j$
is defined recursively as
\begin{align}
\label{mlp_g}
g_{j}(x_{i},\theta_{1:j}) &=
w_{j}h_{j-1}(x_{i},\theta_{1:j-1})+b_{j},\\
\label{mlp_h}
h_{j}(x_{i}, \theta_{1:j}) &=
\phi_{j}(g_{j}(x_{i},\theta_{1:j})),
\end{align}
for $j\in\{1,2,\dots,\rho\}$.
An input data point $x_{i}\in\mathbb{R}^{\kappa_0}$
is passed to the \textit{input layer} $h_{0}(x_{i})=x_{i}$,
yielding vector
$g_{1}(x_{i}, \theta_{1})=w_{1}x_{i}+b_{1}$ 
in the first hidden layer.
The parameters $\theta_{j} = (w_{j}, b_{j})$
at layer $j$
consist of weights $w_{j}$ and
biases $b_{j}$.
The weight matrix $w_{j}$ has $\kappa_{j}$ rows and 
$\kappa_{j-1}$ columns,
while the vector $b_{j}$ of biases has length $\kappa_{j}$.
All weights and biases up to layer $j$ are denoted by
$\theta_{1:j} = (\theta_{1},\theta_{2},\dots,\theta_{j})$.
An \textit{activation} function
$\phi_{j}$
is applied elementwise to
\textit{pre-activation} vector
$g_{j}(x_{i},\theta_{1:j})$,
and returns
\textit{post-activation} vector
$h_{j}(x_{i}, \theta_{1:j})$.
Concatenating all $\theta_{j},~j\in\{1,2,\dots,\rho\}$,
gives a parameter vector $\theta=\theta_{1:\rho}\in\mathbb{R}^n$
of length $n=\sum_{j=1}^{\rho}\kappa_{j}(\kappa_{j-1}+1)$.


\begin{figure}[t]
	\centering
	\includegraphics[width=1\linewidth]{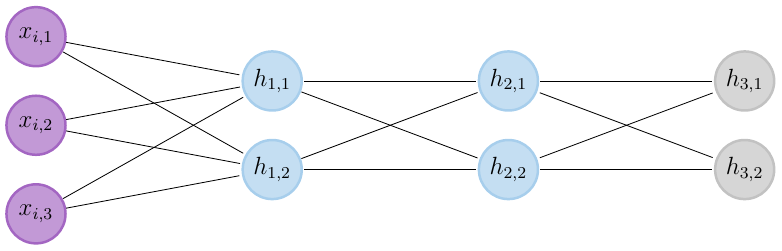}
	\caption{A graph visualization of
		$\mbox{MLP}(3, 2, 2, 2)$.
		Purple, blue and gray nodes
		correspond to input data,
		to hidden layer post-activations
		and to output layer (softmax) post-activations
		used for making predictions.}
	\label{mlp_example}
\end{figure}

$w_{j,k,l}$ denotes the
$(k,l)$-th element of weight matrix $w_{j}$.
Analogously,
$b_{j,k},~x_{i,k},~g_{j,k}$ and $h_{j,k}$ correspond to
the $k$-th coordinate
of bias $b_{j}$,
of input $x_{i}$,
of pre-activation $g_{j}$ and
of post-activation $h_{j}$.

MLPs are typically visualized as graphs. For instance,
Figure~\ref{mlp_example} displays a graph representation of
$\mbox{MLP}(\kappa_0=3,\kappa_1=2,\kappa_2=2,\kappa_3=2)$,
which has
an input layer with $\kappa_0=3$ nodes (purple),
two hidden layers with $\kappa_1=\kappa_2=2$ nodes each (blue),
and an output layer with $\kappa_3=2$ nodes (gray).
Purple nodes indicate observed variables (input data),
whereas blue and gray nodes indicate latent variables
(post-activations).

Let
$\mathcal{D}_{1:s}=\{(x_{i}, y_{i}):~i=1,2,\dots,s\}$
be a training dataset.
Each training data point $(x_i, y_i)$
includes
an input $x_{i}\in\mathbb{R}^{\kappa_0}$
and a discrete output (label)
$y_i\in\{1,2,\dots,\kappa_{\rho}\},~\kappa_{\rho}\ge 2$.
Moreover, let $(x, y)$ be a test point consisting of
an input $x\in\mathbb{R}^{\kappa_0}$ and of a label
$y\in\{1,2,\dots,\kappa_{\rho}\}$.
The \textit{supervised classification problem}
under consideration
is to predict test label $y$ given test input $x$
and training dataset $\mathcal{D}_{1:s}$.
An $\mbox{MLP}(\kappa_{0:\rho})$,
whose output layer
has $\kappa_{\rho}$ nodes and
applies the \textit{softmax activation function} $\phi_{\rho}$,
is used to address this problem.
The softmax activation function
at the output layer expresses as
$\phi_{\rho}(g_{\rho})=\exp{(g_{\rho})} / 
\sum_{k=1}^{\kappa_{\rho}}\exp{(g_{\rho,k})}$.

It is assumed that the training labels
$y_{1:s}=(y_{1},y_{2},\dots,y_{s})$
are outcomes of $s$ independent draws from
a categorical pmf
with event probabilities given by
$\mbox{Pr}(y_{i}=k \vert x_{i},\theta)
= h_{\rho,k}(x_{i},\theta)
= \phi_{\rho}(g_{\rho,k}(x_{i},\theta))$,
where $\theta$ is the set of
$\mbox{MLP}(\kappa_{0:\rho})$
parameters.
It follows that the
likelihood function
for the $\mbox{MLP}(\kappa_{0:\rho})$
model in supervised classification is
\begin{equation}
\label{mc_mlp_lik}
\mathcal{L}(y_{1:s} \vert x_{1:s},\theta)=
\prod_{i=1}^{s}\prod_{k=1}^{\kappa_{\rho}}
(h_{\rho,k}(x_{i},\theta))^{\mathbbm{1}_{\{y_{i}=k\}}},
\end{equation}
where
$x_{1:s}=(x_{1},x_{2},\dots,x_{s})$
are the training inputs
and $\mathbbm{1}$ denotes the indicator function.
Interest is in sampling from the
parameter posterior density
\begin{equation}
\label{param_posterior}
p(\theta \vert x_{1:s}, y_{1:s})
\propto
\mathcal{L}(y_{1:s} \vert x_{1:s}, \theta)
\pi (\theta),
\end{equation}
given
the likelihood function
$\mathcal{L}(y_{1:s} \vert x_{1:s}, \theta)$
of Equation~\eqref{mc_mlp_lik}
and a parameter prior $\pi(\theta)$.
For brevity,
the parameter posterior density
$p(\theta \vert x_{1:s}, y_{1:s})$
is alternatively
denoted by $p(\theta \vert D_{1:s})$.

By integrating out
parameters $\theta$,
the posterior predictive pmf of test label $y$
given test input $x$ and training dataset $\mathcal{D}_{1:s}$
becomes
\begin{equation}
\label{pred_posterior}
p(y \vert x, \mathcal{D}_{1:s}) =
\int 
\mathcal{L}(y \vert x, \theta)
p(\theta \vert \mathcal{D}_{1:s})
d\theta ,
\end{equation}
where $\mathcal{L}$
is the likelihood function
of Equation~\eqref{mc_mlp_lik} evaluated on $(x, y)$,
and $p(\theta \vert \mathcal{D}_{1:s})$ is the 
parameter posterior density
of Equation~\eqref{param_posterior}.
The integral in Equation~\eqref{pred_posterior}
can be approximated via Monte Carlo integration,
yielding the approximate
posterior predictive pmf
\begin{equation}
\label{pred_posterior_approx}
\hat{p}(y \vert x,\mathcal{D}_{1:s}) \simeq
\sum_{t=1}^{v} p(y \vert x, \omega_{t}) ,
\end{equation}
where $(\omega_{1},\omega_{2},\ldots,\omega_{v})$
is a Markov chain realization
obtained from the parameter posterior density
$p ({\theta \vert D_{1:s}})$.
Maximizing the
approximate posterior predictive pmf
$\hat{p}(y \vert x,D_{1:s})$ of
Equation~\eqref{pred_posterior_approx}
yields the prediction
\begin{equation}
\label{multi_class_pred}
\hat{y} =
\argmax_{y} {
	\{\hat{p}(y \vert x,D_{1:s})\}
}
\end{equation}
for test label $y\in\{1,2,\dots,\kappa_{\rho}\}$.

The likelihood function for an
MLP model with $\kappa_{\rho}\ge 2$ output layer nodes,
as stated in Equation~\eqref{mc_mlp_lik},
is suited for multiclass classification with
$\kappa_{\rho}$ classes.
For binary classification,
which involves two classes,
Equation~\eqref{mc_mlp_lik}
is related to an
MLP with $\kappa_{\rho}=2$ output layer nodes.
There is an alternative likelihood function based on
an MLP model with a single output layer node,
which can be used for binary classification;
see~\cite{papamarkou2022} for details.

\subsection{Blocked Gibbs sampling}
\label{bgs}

A blocked Gibbs sampling algorithm samples
groups (blocks) of two or more parameters
conditioned on all other other parameters,
rather than sampling each parameter individually.
The choice of parameter groups affects
the rate of convergence~\citep{roberts1997}.
For instance,
breaking down the parameter space into
statistically independent groups of correlated parameters
speeds up convergence.

To sample
from the parameter posterior density
$p(\theta \vert \mathcal{D}_{1:s})$
of an $\mbox{MLP}(\kappa_{0:\rho})$ model
fitted to a training dataset $\mathcal{D}_{1:s}$,
a blocked Gibbs sampling algorithm
utilizes a partition
$\{\theta_{z(1)},\theta_{z(2)},\ldots,\theta_{z(m)}\}$
of the MLP parameters
$\theta=(\theta_1,\theta_2,\ldots,\theta_n)$.
Due to partitioning
$\{\theta_1,\theta_2\ldots,\theta_n\}$,
the parameter subsets
$\theta_{z(1)},\theta_{z(2)},\ldots,\theta_{z(m)}$
are pairwise disjoint and satisfy
$\displaystyle\cup_{q=1}^{m}\theta_{z(q)}=
\{\theta_1,\theta_2,\ldots,\theta_n\},~m\le n$.
Without loss of generality,
it is assumed that each
subset $\theta_{z(q)}$ of $\theta$
is totally ordered.
For any $(c, q)$ such that $1\le c \le q\le m$,
the shorthand notation
$\theta_{z(c):z(q)}=
(\theta_{z(c)},\theta_{z(c+1)},\ldots,\theta_{z(q)})$
is used hereafter.
So, the vector 
$\theta_{z(1):z(m)}$ is a permutation of $\theta$.

Under such a setup,~\ref{app:bg}
summarizes blocked Gibbs sampling.
At iteration $t$,
for each $q\in\{1,2,\ldots,m\}$,
a blocked Gibbs sampling algorithm
draws a sample
$\theta_{z(q)}^{(t)}$
of parameter group
$\theta_{z(q)}$
from the corresponding conditional density
$p(\theta_{z(q)} \vert
\theta_{z(1):z(q-1)}^{(t)},
\theta_{z(q+1):z(m)}^{(t-1)},
\mathcal{D}_{1:s})$.
To put it another way,
at each iteration,
a sample is drawn from the conditional density of
each parameter group
conditioned on the most recent values
of the other
parameter groups
and on the training dataset.



\section{Methodology}
\label{methodology}





This section introduces a blocked Gibbs sampling algorithm
for MLPs in supervised classification.
MLP parameter blocks are determined by linking
parameters to MLP nodes,
as elaborated in Subsections
~\ref{ss:bgs_via_cs} and~\ref{ss:fnbg}
and as exemplified in Subsections
~\ref{ss:toy_ex} and~\ref{ss:mnist_ex}.

Minibatching and parameter blocking render
the proposed Gibbs sampler possible.
Blocked Gibbs sampling is typically motivated by
increased rates of convergence attained via
near-optimal or optimal parameter groupings.
Although low speed of convergence is
a problem with MCMC in deep learning,
near-zero acceptance rates constitute a
more immediate problem.
In other words,
no mixing is a more pressing issue than
slow mixing.
By updating a small block of parameters at a time
instead of updating all parameters
via a single step,
each block-specific acceptance rate moves away from zero.
So, minibatch blocked Gibbs sampling provides
a workaround for vanishing acceptance rates in deep learning.
Of course there is no free lunch;
increased acceptance rates come at a computational price
per Gibbs step,
which consists of additional conditional density sampling sub-steps.

In typical SG-LD-within-Gibbs and
SG-HMC-within-Gibbs software implementations,
one block of parameters
is formed for each MLP layer
(see Subsection~\ref{subsec:gibbs_sampling}).
A caveat to grouping parameters by MLP layer
is that parameter block sizes
depend on layer widths.
Hence, a parameter block can be large,
containing hundreds or thousands of parameters,
in which case the problem of low acceptance rate
is not resolved.
The blocked Gibbs sampler of this paper
groups parameters by MLP node and allows to further
partition parameters into smaller blocks
within each node,
thus controlling the number of parameters per block.

While structured SG-MCMC~\citep{alexos2022}
also splits the parameter space into blocks,
it uses the parameter blocks to factorize
a variational posterior density. Hence,
structured SG-MCMC aims to solve
the low acceptance and slow mixing problems
by factorizing an approximate parameter posterior density.
The blocked Gibbs sampler herein
factorizes the exact parameter posterior density,
relying on finer parameter grouping.
Minibatching, which is the only type of approximation
employed by the blocked Gibbs sampler of this paper,
is an approximation related to the data,
not to the MLP model.

The finer node-blocked Gibbs sampler
for feedforward neural networks,
as presently conceived here, is
a minibatch MH-within-Gibbs sampler.
The main idea is
to update a relatively small block
of neural network parameters,
thus making it possible to accept states proposed by minibatch MH.
Due to taking minibatch MH sampling steps per block of parameters,
the sampler is gradient-free.
Such a gradient-free approach has been chosen
to cap the computational cost.
Subject to availability of computing resources,
SG-LD or SG-HMC sampling steps
can be taken instead of minibatch MH sampling steps.

\subsection{Metropolis inside blocks}
\label{ss:bgs_via_cs}

Blocked Gibbs sampling
raises the question
how to sample each parameter block
from its conditional density.
Such conditional densities for MLPs
are not available in closed form.
Instead, a single Metropolis-Hastings step
can be taken to draw a sample from
a conditional density.
In this case,
the resulting blocked Gibbs sampling algorithm
is known as Metropolis-within-blocked-Gibbs (MWBG)
sampling.

At iteration $t$ of MWBG,
a candidate state $\theta_{z(q)}^{\star}$
for parameter block $\theta_{z(q)}$
can be sampled from an isotropic normal proposal density
$\mathcal{N}(\theta_{z(q)}^{(t-1)}, \sigma_q^2 I_q)$
centered at state $\theta_{z(q)}^{(t-1)}$ of iteration $t-1$,
where $I_q$ is the
$\vert \theta_{z(q)} \vert \times \vert \theta_{z(q)} \vert$
identity matrix,
$\vert \theta_{z(q)} \vert$
is the number of parameters
in block $\theta_{z(q)}$, and
$\sigma_q^2 > 0$ is the proposal variance
for block $\theta_{z(q)}$.
\tp{
The acceptance probability
$a (\theta_{z(q)}^{\star}, \theta_{z(q)}^{(t-1)})$
of candidate state $\theta_{z(q)}^{\star}$
is given by
\begin{equation}
\label{mwbg_ac_e_eq}
\begin{split}
& a (\theta_{z(q)}^{\star}, \theta_{z(q)}^{(t-1)}) = \\
&
\min{\left\{
\displaystyle
\frac{\pi (\theta_{z(q)}^{\star})  \exp{\left(
		\mathcal{E}(\theta^{(t-1)}, \mathcal{D}_{1:s})
		\right)} }
{\pi (\theta_{z(q)}^{(t-1)})
\exp{\left(
\mathcal{E}(\theta^{\star}, \mathcal{D}_{1:s})
	\right)} },
1\right\}},
\end{split}
\end{equation}
where $\mathcal{E}$ denotes the cross-entropy loss function.
More details for the acceptance probability
$a (\theta_{z(q)}^{\star}, \theta_{z(q)}^{(t-1)})$
are available in~\ref{app:bg}.
}

\begin{algorithm*}[t]
	\caption{Metropolis-within-blocked-Gibbs
		(MWBG) sampling
		based on cross-entropy}
	\label{bmwg_e}
	\begin{algorithmic}[1]
		\State{\textbf{Input}:
			training dataset $\mathcal{D}_{1:s}$
		}
		\State{\textbf{Input}:
			initial state
			$ \theta_{z(1):z(m)}^{(0)}$
		}
		\State{\textbf{Input}:
			proposal variances
			$(\sigma_1^2, \ldots, \sigma_m^2)$
			across blocks
		}
		\State{\textbf{Input}:
			number of Gibbs sampling iterations
			$v$
		}
		
		$~$
		
		\For{$t=1,\ldots,v$
		}
		\For{$q=1,\ldots,m$
		}
		\State{Draw
			$\theta_{z(q)}^{\star}\sim
			\mathcal{N}(\theta_{z(q)}^{(t-1)}, \sigma_q^2 I_{q})
			$
		}
		\State{Compute
			$a (\theta_{z(q)}^{\star}, \theta_{z(q)}^{(t-1)}) =
\min{\left\{
	\displaystyle
	\frac{\pi (\theta_{z(q)}^{\star})  \exp{\left(
			\mathcal{E}(\theta^{(t-1)}, \mathcal{D}_{1:s})
			\right)} }
	{\pi (\theta_{z(q)}^{(t-1)})
		\exp{\left(
			\mathcal{E}(\theta^{\star}, \mathcal{D}_{1:s})
			\right)} },
	1\right\}}
			$
		}
		\State{Draw $u\sim\mathcal{U}(0,1)$}
		\If{$u \le a (\theta_{z(q)}^{\star}, \theta_{z(q)}^{(t-1)})$}
		\State{Set $\theta_{z(q)}^{(t)} = \theta_{z(q)}^{\star}$}
		\Else{}
		\State{Set $\theta_{z(q)}^{(t)} = \theta_{z(q)}^{(t-1)}$}
		\EndIf
		\EndFor
		\EndFor
	\end{algorithmic}
\end{algorithm*}

Algorithm~\ref{bmwg_e} summarizes
exact MWBG sampling.
To make Algorithm~\ref{bmwg_e}
amenable to big data,
minibatching can be used
by replacing all instances of $\mathcal{D}_{1:s}$
\tp{with batches (strict subsets of $\mathcal{D}_{1:s}$);
the resulting approximate MCMC algorithm
is termed `minibatch MWBG sampling'.}

\subsection{Finer blocks}
\label{ss:fnbg}



Big data and big models challenge
the adaptation of MCMC sampling methods in deep learning.
Minibatching provides a way of applying MCMC to big data.
It is less clear how to apply MCMC
to big neural network models,
containing thousands or millions of parameters.
Minibatch MWBG sampling proposes a way forward
by drawing an analogy between subsetting data
and subsetting model parameters.
As data batches reduce the dimensionality of
data per Gibbs sampling iteration,
parameter blocks
reduce the dimensionality of
parameters per Metropolis-within-Gibbs update.

In an $\mbox{MLP}(\kappa_{0:\rho})$ with $n$ parameters,
layer $j$ contains $\kappa_j (\kappa_{j-1}+1)$ parameters,
of which $\kappa_j \kappa_{j-1}$ are weights and
$\kappa_j$ are biases.
So, if parameters are grouped by layer,
then the block of layer $j$ contains
$\kappa_j (\kappa_{j-1}+1)$ parameters.
The number of parameters in the block of layer $j$
grows linearly with the number $\kappa_j$ of nodes in layer $j$
as well as linearly with the number $\kappa_{j-1}$
of nodes in layer $j-1$.

If parameters are grouped by node,
then each node block in layer $j$ contains
$\kappa_{j-1}+1$,
of which $\kappa_{j-1}$ are weights and one is bias.
The number of parameters in a node block in layer $j$
does not depend on the number $\kappa_j$ of nodes in layer $j$,
but it grows linearly with the number
$\kappa_{j-1}$ of nodes in layer $j-1$.
MWBG sampling (Algorithm~\ref{bmwg_e})
based on parameter grouping by MLP node is
termed `(Metropolis-within-)node-blocked-Gibbs
(NBG) sampling'.

Finer parameter blocks of smaller size
can be generated by splitting the 
$\kappa_{j-1}+1$ parameters of a node in layer $j$
into $\beta_j$ subgroups.
In this case, each finer parameter block
in each node in layer $j$ contains
$(\kappa_{j-1}+1)/\beta_j$ parameters.
If hyperparameter $\beta_j$ is chosen to be a
linear function of $\kappa_{j-1}$,
then the number of parameters
per finer block per node in layer $j$
depends neither on the number $\kappa_j$
of nodes in layer $j$
nor on the number $\kappa_{j-1}$
of nodes in layer $j-1$.
MWBG sampling (Algorithm~\ref{bmwg_e})
based on finer parameter grouping per node is
termed `(Metropolis-within-)finer-node-blocked-Gibbs
(FNBG) sampling'.

\tp{Parameter blocks of smaller size
increase both the acceptance rate per block and
the computational complexity of FNBG sampling.
Thus, the number of parameters per block regulates
the trade-off between acceptance rates and computational complexity.
As a practical guideline, the number of parameters per block
can be tuned by reducing it incrementally until non-vanishing
acceptance rates are attained in order to make sampling possible.
The question of optimal parameter block size for sampling
is analogous to the question of optimal learning rate for stochastic optimization.
Both of these questions pose hyperparameter optimization problems,
which can be approached primarily from an engineering perspective
in lieu of theoretical solutions.}

\subsection{Finer blocks: toy example}
\label{ss:toy_ex}


The $\mbox{MLP}(3,2,2,2)$ architecture
shown in Figure~\ref{mlp_example}
provides a toy example that showcases
layer-based, node-based and finer node-based parameter grouping
(more briefly termed
`layer-blocking', `node-blocking' and `finer node-blocking').
It is reminded that finer node-based grouping
refers to parameter grouping into smaller blocks within each node.
Figure~\ref{fnbg} shows
the directed acyclic graph (DAG)
representation of $\mbox{MLP}(3,2,2,2)$,
augmenting Figure~\ref{mlp_example}
with parameter annotations
and with a layer consisting of a single node
that represents label $y_i$.
Yellow shapes indicate parameters;
yellow circles and boxes
correspond to biases and weights.
Yellow boxes adhere to
expository visual conventions of plate models,
with each box representing a set of weights.
Purple nodes indicate observed variables (input and output data),
whereas blue and gray nodes indicate latent variables
(post-activations).

\begin{figure*}[t]
	\centering
	\includegraphics[width=0.7\linewidth]{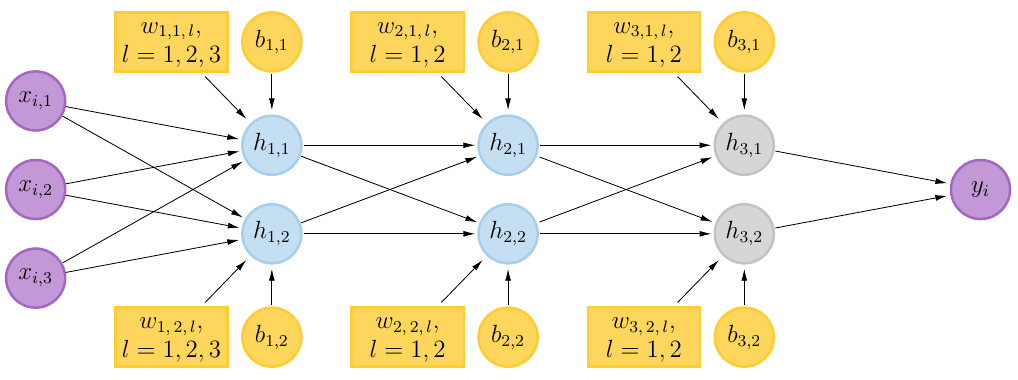}
	\caption{Visual demonstration of
		node-based parameter blocking
		for the $\mbox{MLP}(3,2,2,2)$ architecture.
		The MLP is expressed as a DAG.
		Yellow nodes and yellow plates
		correspond to biases and weights.
		Each of the blue hidden layer nodes and of
		the gray output layer nodes is assigned
		a parameter block of yellow parent nodes in the DAG.}
	\label{fnbg}
\end{figure*}


Layer-blocking partitions 
the set of $20$ parameters of $\mbox{MLP}(3,2,2,2)$
to three blocks
$\theta_{z(1)},~\theta_{z(2)},~\theta_{z(3)},$
which contain
$\vert\theta_{z(1)}\vert=8,~
\vert\theta_{z(2)}\vert=6,~
\vert\theta_{z(3)}\vert=6$
parameters.
For instance, the first hidden layer induces block
$\theta_{z(1)}=(w_{1,1,1:3}, b_{1,1}, w_{1,2,1:3}, b_{1,2})$,
where
$w_{j,k,1:l}=(w_{j,k,1},w_{j,k,2},\ldots,w_{j,k,l})$.

Node-blocking partitions
the set of $20$ parameters of $\mbox{MLP}(3,2,2,2)$
to six blocks,
as many as the number of hidden and output layer nodes.
Each blue or gray node
in a hidden layer or in the output layer
has its own distinct set of yellow weight and bias parents.
Parameters are grouped according to shared parenthood.
For instance, the parameters of block
$\theta_{z(1)}=(w_{1,1,1:3}, b_{1,1})$,
have node $h_{1,1}$ as a common child.

Acceptance probabilities for parameter blocks
require likelihood function evaluations.
It is not possible to factorize conditional densities
to achieve more computationally efficient block updates.
For instance, as it can be seen in Figure~\ref{fnbg},
changes in block $\theta_{z(1)}=(w_{1,1,1:3}, b_{1,1})$
induced by node $h_{1,1}$ in layer $1$
propagate through subsequent layers
due to the hierarchical MLP structure,
thus prohibiting a factorization of conditional density
$p(\theta_{z(1)} \vert \theta_{z(2):z(6)}, \mathcal{D}_{1:s})$.
More formally, each pair of node-based parameter blocks
forms a v-structure, having label $y_i$ (purple node)
as a descendant. Since training label $y_i$ is observed,
such v-structures are activated,
and therefore any two node-based parameter blocks
are not conditionally independent given label $y_i$.

As a demonstration of finer node-blocking
for $\mbox{MLP}(3,2,2,2)$,
set $\beta_{1}=2$ in layer $1$.
For $\beta_{1}=2$, blocks
$\theta_{z(1)}=w_{1,1,1:2}$ and
$\theta_{z(2)}=(w_{1,1,3}, b_{1,1})$
are generated within node $h_{1,1}$.
Similarly, blocks 
$\theta_{z(3)}=w_{1,2,1:2}$ and
$\theta_{z(4)}=(w_{1,2,3}, b_{1,2})$
are generated within node $h_{1,2}$.

To recap on this toy example,
layer-based grouping produces
a single block of eight parameters in layer $1$,
node-based grouping produces
two blocks of four parameters each in layer $1$,
and a case of finer node-based grouping produces
four blocks of two parameters each in layer $1$.
It is thus illustrated that finer blocks per node
provide a way to reduce the number of parameters
per Gibbs sampling block.





\subsection{Finer blocks: MNIST example}
\label{ss:mnist_ex}

After having used $\mbox{MLP}(3,2,2,2)$ as a toy example
to describe the basics of finer node-blocking,
the wider $\mbox{MLP}(784,10,10,10,10)$ architecture
is utilized to elaborate on the practical relevance
of smaller blocks per node.
An $\mbox{MLP}(784,10,10,10,10)$ is fitted
to the MNIST (and FMNIST) training dataset
in Section~\ref{experiments}.
An $\mbox{MLP}(784,10,10,10,10)$
contains $8180$ parameters,
of which $7850,~110,~110$ and $110$
have children nodes
in the first, second, third hidden layer
and output layer, respectively.

So, layer-blocking for $\mbox{MLP}(784,10,10,10,10)$
involves four parameter blocks
$\theta_{z(1)},~\theta_{z(2)},~\theta_{z(3)},~\theta_{z(4)}$
of sizes
$\vert\theta_{z(1)}\vert=7850,~
\vert\theta_{z(2)}\vert=
\vert\theta_{z(3)}\vert=
\vert\theta_{z(4)}\vert=110$.
Metropolis-within-Gibbs updates
for block $\theta_{z(1)}$
have zero or near-zero acceptance rate
due to the large block size of $\vert\theta_{z(1)}\vert=7850$.
Although each of blocks
$\theta_{z(2)},~\theta_{z(3)},~\theta_{z(4)}$
has nearly two orders of magnitude smaller size than $\theta_{z(1)}$,
a block size of
$\vert\theta_{z(2)}\vert=
\vert\theta_{z(3)}\vert=
\vert\theta_{z(4)}\vert=110$
might be large enough to yield Metropolis-within-Gibbs updates
with prohibitively low acceptance rate.

Node-blocking for $\mbox{MLP}(784,10,10,10,10)$
entails a block of $785$ parameters for each node
in the first hidden layer,
and a block of $11$ parameters for each node
in the second and third hidden layer
and in the output layer.
Thus, node-blocking addresses
the low acceptance rate problem
related to large parameter blocks
for block updates in all layers
apart from the first hidden layer.

There is no practical need to carry out finer node-blocking
in nodes belonging to the second or third hidden layer
or to the output layer of $\mbox{MLP}(784,10,10,10,10)$,
since each block in these layers
contains only $11$ parameters based on node-blocking.
On the other hand, finer node-blocking is useful
in nodes belonging to the first hidden layer,
since each block related to such nodes contains
a large number of $785$ parameters.
By setting $\beta_1=10$,
smaller blocks (each consisting of $78$ or $79$ parameters)
are generated in the first hidden layer.
So, finer node-blocking disentangles
block sizes in the first hidden layer
from input data dimensions,
making it possible to decrease block sizes
and to consequently increase acceptance rates.

\section{Experiments}
\label{experiments}

Minibatch FNBG sampling is put into practice to make empirical observations
about several characteristics
of approximate MCMC
in deep learning.
In the experiments of this section,
parameters of MLPs are sampled.
Three datasets are used, namely
a simulated noisy version of exclusive-or~\citep{papamarkou2022},
MNIST~\citep{lecun1998} and
fashion MNIST~\citep{xiao2017}.
For brevity,
exclusive-or and fashion MNIST
are abbreviated to XOR and FMNIST.
Table~\ref{data_models_table}
displays the correspondence between used datasets and fitted MLPs.

\begin{table*}[t]
	\centering
	\caption{Datasets used in the experiments
		and MLPs fitted to these datasets.
		Training and test dataset sample sizes
		as well as MLP parameter dimensions
		are shown.}
	\label{data_models_table}
	\begin{tabular}{lrrlr}
		\hline
		\multicolumn{3}{c}{Dataset} & \multicolumn{2}{c}{Neural network}\\ \hline
		\multicolumn{1}{c}{\multirow{2}{*}{Name}} &
		\multicolumn{2}{c}{Sample size} & 
		\multicolumn{1}{c}{\multirow{2}{*}{Architecture}} &
		\multicolumn{1}{c}{\multirow{2}{*}{\# parameters}}\\
		\cline{2-3}
		&
		\multicolumn{1}{c}{Training} &
		\multicolumn{1}{c}{Test}
		&
		&
		\\
		\hline
		Noisy XOR &  $5000$ &  $1200$ & $\mbox{MLP}(2,2,1)$           &     $9$ \\ \hline
		Noisy XOR &  $5000$ &  $1200$ & $\mbox{MLP}(2,2,2,2,2,2,2,1)$ &    $39$ \\ \hline
		MNIST     & $60000$ & $10000$ & $\mbox{MLP}(784,10,10,10,10)$ & $8180$ \\ \hline
		FMNIST    & $60000$ & $10000$ & $\mbox{MLP}(784,10,10,10,10)$ & $8180$ \\ \hline
	\end{tabular}
\end{table*}


The noisy XOR training and test datasets
are visualized
in Figure~\ref{noisy_xor_scatter}
of~\ref{app:noisy_xor}.
Random perturbations
of $(0, 0)$ and of $(1, 1)$,
corresponding to gray and yellow points,
are mapped to $0$ (circles).
Moreover, random perturbations
of $(0, 1)$ and of $(1, 0)$,
corresponding to purple and blue points,
are mapped to $1$ (triangles).
More information about
the simulation of noisy XOR
can be found in~\cite{papamarkou2022}.

Each MNIST and FMNIST image is firstly reshaped,
by converting it from a $28\times 28$ matrix
to a vector of length $784 = 28\times 28$,
and it is subsequently standardized.
This image reshaping explains why
the $\mbox{MLP}(784, 10, 10, 10, 10)$ model,
which is fitted to MNIST and FMNIST,
has an input layer width of $784$.

\subsection{Experimental configuration}

Binary classification for noisy XOR
is performed via the likelihood function
based on binary cross-entropy,
as described in~\cite{papamarkou2022}.
Multiclass classification for MNIST and FMNIST
is performed via the likelihood function
given by Equation~\eqref{mc_mlp_lik},
which is based on cross-entropy.

The sigmoid activation function is applied at each hidden layer
of each MLP of Table~\ref{data_models_table}.
Furthermore, the
sigmoid activation function is also applied
at the output layer of
$\mbox{MLP}(2, 2, 1)$ and
of $\mbox{MLP}(2, 2, 2, 2, 2, 2, 2, 1)$,
conforming to the employed likelihood function
for binary classification.
The softmax activation function is applied
at the output layer of
$\mbox{MLP}(784, 10, 10, 10, 10)$,
in accordance with likelihood function
\eqref{mc_mlp_lik}
for multiclass classification.
The same $\mbox{MLP}(784, 10, 10, 10, 10)$ model
is fitted to the MNIST and FMNIST datasets.

A normal prior
$\pi(\theta)\sim\mathcal{N}(0, 10 I)$ is adopted
for the parameters $\theta\in\mathbb{R}^n$
of each MLP model shown in
Table \ref{data_models_table}.
\tp{Thus, a relatively high variance (equal to $10$)
is assigned a priori to each parameter.}

NBG sampling is run upon fitting
$\mbox{MLP}(2, 2, 1)$ and
$\mbox{MLP}(2, 2, 2, 2, 2, 2, 2, 1)$
to the noisy XOR training set,
while FNBG sampling is run upon fitting
$\mbox{MLP}(784, 10, 10, 10, 10)$
to the MNIST and FMNIST training sets.
So, parameters are grouped by node in
$\mbox{MLP}(2, 2, 1)$ and
$\mbox{MLP}(2, 2, 2, 2, 2, 2, 2, 1)$,
whereas multiple parameter groups per node
are formed in the first hidden layer of
$\mbox{MLP}(784, 10, 10, 10, 10)$
as elaborated in Subsection~\ref{ss:mnist_ex}.
Parameters are grouped by node from
the second hidden layer onwards in
$\mbox{MLP}(784, 10, 10, 10, 10)$.
All three MLPs of Table~\ref{data_models_table}
are relatively shallow neural networks.
However, $\mbox{MLP}(784, 10, 10, 10, 10)$
has two orders of magnitude larger input layer width
in comparison to
$\mbox{MLP}(2, 2, 1)$ and
$\mbox{MLP}(2, 2, 2, 2, 2, 2, 2, 1)$.
So, the higher dimension
of MNIST and FMNIST input data necessitates
finer node-blocking in the first hidden layer of
$\mbox{MLP}(784, 10, 10, 10, 10)$.
On the other hand, the smaller dimension
of noisy XOR input data
implies that finer blocks per node
are not required in the first hidden layer of
$\mbox{MLP}(2, 2, 1)$ or of
$\mbox{MLP}(2, 2, 2, 2, 2, 2, 2, 1)$.

A normal proposal density is chosen for each parameter block.
The variance of each proposal density is a hyperparameter,
thus enabling to tune the magnitude of proposal steps
separately for each parameter block.
\tp{Preliminary FNBG pilot runs have been carried out in order
to tune the proposal variances. During this pre-training stage,
the proposal variances have been set initially
to a single relatively high value
across all parameter blocks.
Subsequently, the proposal variances of blocks in each hidden layer
have been reduced to smaller values in deeper layers
until non-vanishing acceptance rates have been attained.}

$m=10$ Markov chains are realized for noisy XOR,
whereas $m=1$ chain is realized for each of MNIST and FMNIST
due to computational resource limitations.
$110000$ iterations are run per chain realization,
$10000$ of which are discarded as burn-in.
Thereby, $v=100000$ post-burnin iterations
are retained per chain realization.
Acceptance rates are computed from all
$100000$ post-burnin iterations per chain.

Monte Carlo approximations of posterior predictive pmfs
are computed
according to Equation~\eqref{pred_posterior_approx}
for each data point of each test set.
To reduce the computational cost,
the last $v=10000$ iterations of each realized chain
are used in
Equation~\eqref{pred_posterior_approx}.

Predictions for noisy XOR
are made using the binary classification rule
mentioned in~\cite{papamarkou2022}.
Predictions for MNIST and for FMNIST
are made using the multiclass classification rule
specified by Equation~\eqref{multi_class_pred}.
Given a single chain realization based on a training set,
predictions are made for every point
in the corresponding test set;
the predictive accuracy is then computed as
the number of correct predictions
over the total number of points in the test set.
For the noisy XOR test set,
the mean of predictive accuracies
across the $m=10$ realized chains
is reported.
For the MNIST and FMNIST test sets,
the predictive accuracy
based on the
corresponding single chain realization ($m=1$)
is reported.

\subsection{Exact vs approximate MCMC}

An illustrative comparison
between approximate and exact NBG sampling is made
in terms of acceptance rate, predictive accuracy and runtime.
The comparison between approximate and exact NBG sampling
is carried out in the context of noisy XOR only,
since exact MCMC is not feasible for the MNIST and FMNIST examples
due to vanishing acceptance rates and high computational requirements.

$\mbox{MLP}(2, 2, 2, 2, 2, 2, 2, 1)$
is fitted to the noisy XOR training set
under four scenarios.
For scenario $1$,
approximate NBG sampling is run with a batch size of $100$
to simulate
$m=10$ chains.
For scenario $2$,
\tp{exact NBG is run
to simulate
$10$ chains}.
For scenario $3$,
exact NBG is run until $10$ chains are obtained,
each having an acceptance rate $\ge 5\%$.
For scenario $4$,
exact NBG is run until $10$ chains are acquired,
each with an acceptance rate $\ge 20\%$.
$11$ and $23$ chains have been run in total
under scenarios $3$ and $4$, respectively,
to get $10$ chains
that satisfy the acceptance rate lower bounds
in each scenario.

\tp{It is not suggested to develop a sampling algorithm
that relies on some acceptance rate threshold as a criterion for chain retention,
since such a criterion would introduce bias
in the estimation of the target parameter posterior density.
The purpose of this experiment is to showcase that the avoidance
of prohibitively low acceptance rates enables
the generation of chains with predictive capacity.}

For approximate NBG sampling (scenario $1$),
the proposal variance is set to $0.04$.
For the three exact NBG sampling scenarios,
the proposal variance is lowered to $0.001$
in order to mitigate decreased acceptance rates
in the presence of increased sample size
($5000$ training data points)
relatively to the batch size of $100$
used in approximate sampling.

\begin{figure*}[t]
	\begin{subfigure}{1\textwidth}
		\centering
		\includegraphics[width=1\linewidth]{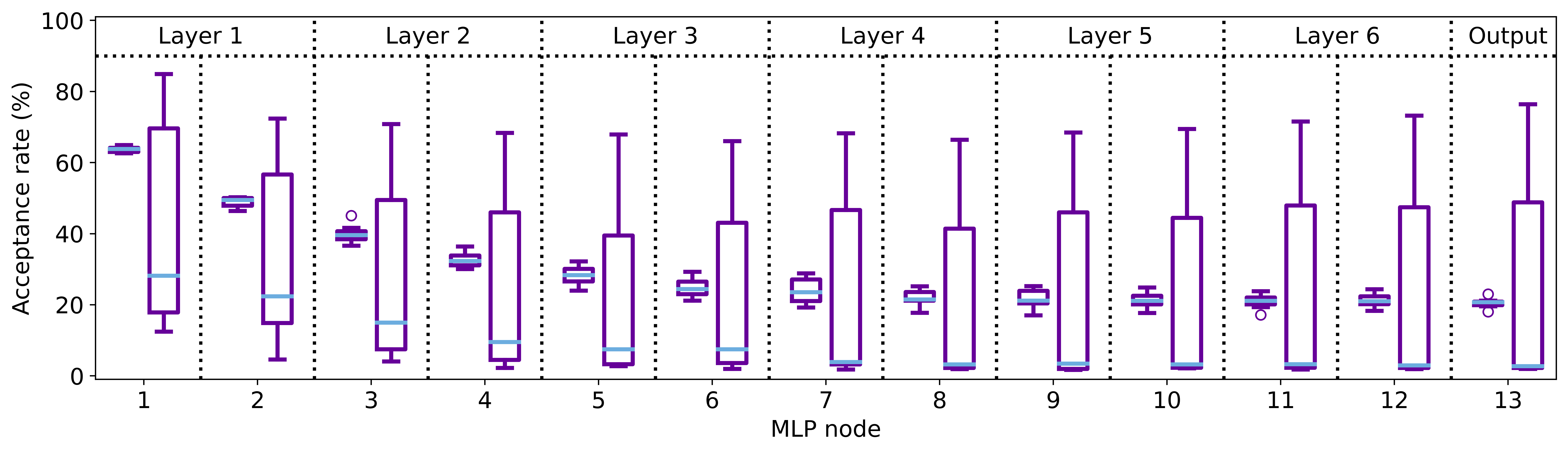}
		\caption{Acceptance rate boxplots.
			The left and right boxplot in each pair correspond to approximate and exact NBG.}
		\label{exact_vs_approx_acc_rate_boxplots}
	\end{subfigure}\\
	\begin{subfigure}{.491\textwidth}
		\centering
		\includegraphics[width=1\linewidth]{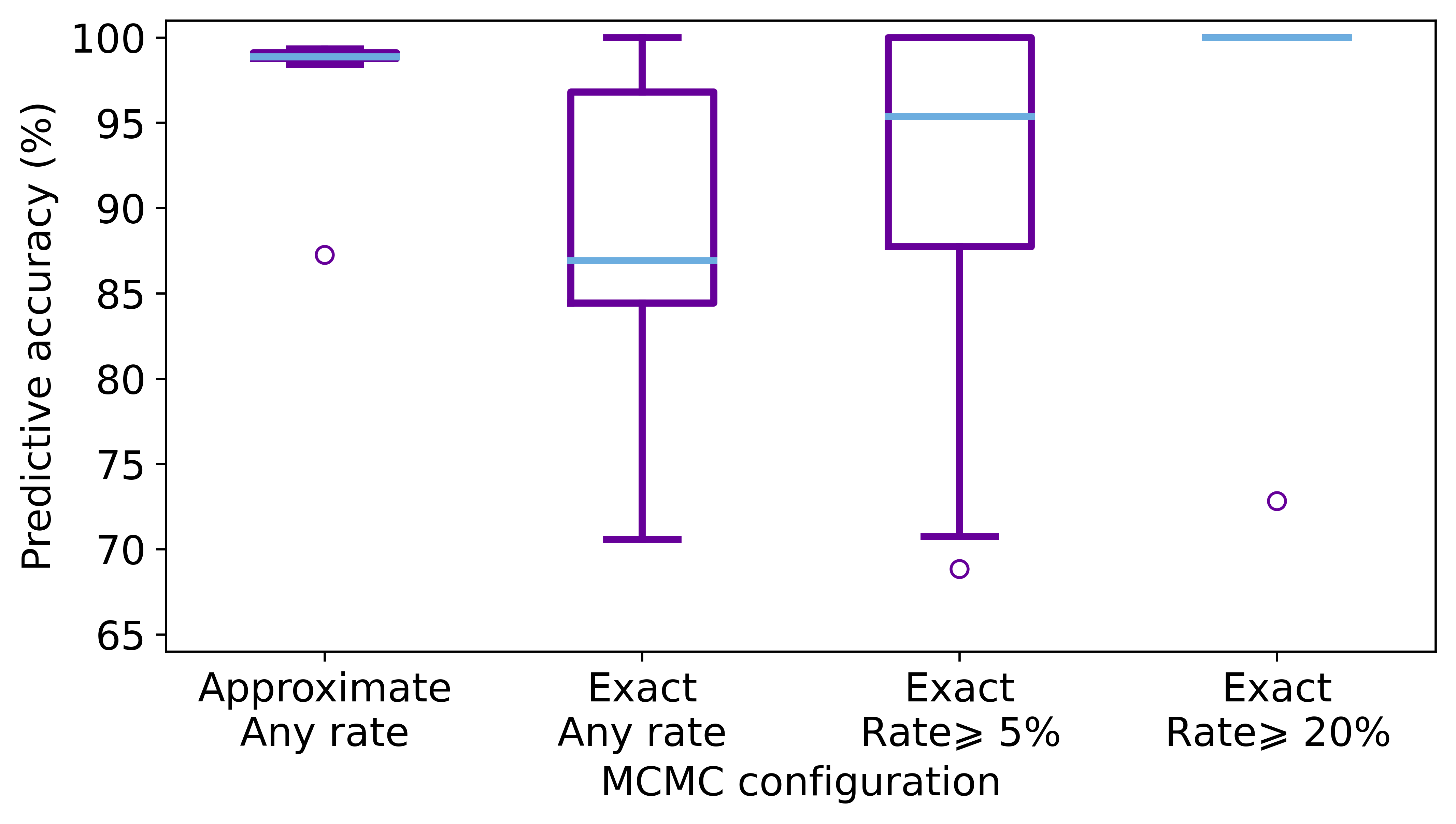}
		\caption{Predictive accuracy boxplots.}
		\label{exact_vs_approx_pred_acc_boxplots}
	\end{subfigure}
	\begin{subfigure}{.491\textwidth}
		\centering
		\includegraphics[width=1\linewidth]{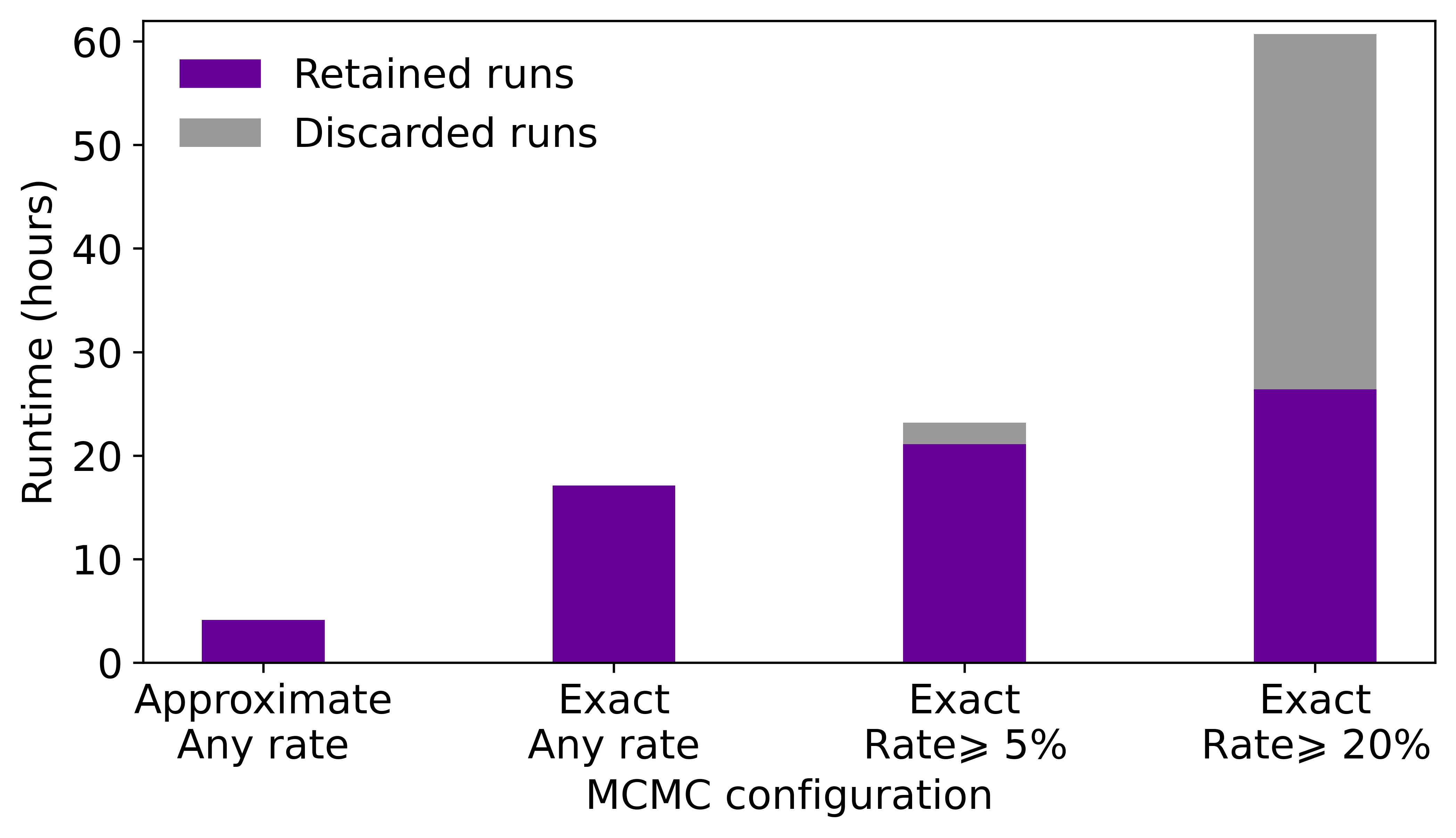}
		\caption{Runtime barplot.}
		\label{exact_vs_approx_time_barplots}
	\end{subfigure}
	\caption{A comparison between approximate and exact NBG sampling.
		$\mbox{MLP}(2, 2, 2, 2, 2, 2, 2, 1)$ is fitted
		to noisy XOR under four scenarios,
		acquiring $10$ chains per scenario.
		Scenario $1$: approximate NBG with a batch size of $100$.
		Scenario $2$: exact NBG.
		Scenario $3$: exact NBG with acceptance rate $\ge 5\%$.
		Scenario $4$: exact NBG with acceptance rate $\ge 20\%$.
		Chains with acceptance rates below $5\%$ in scenario $3$
		and below $20\%$ in scenario $4$
		are discarded until $10$ chains are attained in each case.
	}
	\label{exact_vs_approx_plots}
\end{figure*}

Figure~\ref{exact_vs_approx_acc_rate_boxplots}
displays boxplots of node-specific acceptance rates
for approximate and exact NBG sampling
without lower bound conditions on acceptance rates
(scenarios $1$ and $2$).
A pair of boxplots is shown for each of the $13$ nodes
in the six hidden layers and one output layer of
$\mbox{MLP}(2, 2, 2, 2, 2, 2, 2, 1)$.
The left and right boxplots per pair
correspond to approximate and exact NBG sampling.
Blue lines represent medians.

Three empirical observations are drawn from
Figure~\ref{exact_vs_approx_acc_rate_boxplots}.
First of all,
approximate NBG
attains higher acceptance rates
than exact NBG
according to the (blue) medians,
despite setting higher proposal variance
in the former in comparison to the latter
($0.04$ and $0.001$, respectively).
Secondly,
approximate NBG attains less volatile acceptance rates
than exact NBG
as seen from the boxplot interquartile ranges.
Acceptance rates for exact NBG range
from near $0\%$ to about $50\%$
as neural network depth increases,
exhibiting lack of stability
due to entrapment in local modes
in some chain realizations.
Thirdly, acceptance rates decrease
as depth increases.
For instance,
exact NBG yields median acceptance rates of
$63.83\%$ and $20.72\%$
in nodes $1$ and $13$, respectively.
The attenuation of acceptance rate with depth
is further discussed in Subsection~\ref{depth_and_acceptance}.

Figure~\ref{exact_vs_approx_pred_acc_boxplots}
shows boxplots of predictive accuracies
for the four scenarios under consideration.
Approximate NBG has a median predictive accuracy of $98.88\%$,
with interquartile range concentrated around the median
and with a single outlier ($87.25\%$) in $10$ chain realizations.
Exact NBG without conditions on acceptance rate
and exact NBG conditioned on acceptance rate $\ge 5\%$
have lower median predictive accuracies
($86.92\%$ and $95.38\%$)
and higher interquartile ranges than exact NBG.
Exact NBG conditioned on acceptance rate $\ge 20\%$
attains a median predictive accuracy of $100\%$;
nine out of $10$ chain realizations yield $100\%$ accuracy,
and one chain gives an outlier accuracy of $72.83\%$.
The overall conclusion is that approximate NBG
retains a predictive advantage over exact NBG,
since minibatch sampling ensures consistency
in terms of high predictive accuracy and
reduced predictive variability.
Exact NBG conditioned on higher acceptance rates
can yield near-perfect predictive accuracy
in the low parameter and data dimensions
of the toy noisy XOR example,
but stability and computational issues arise,
as many chains with near-zero acceptance rates are discarded
before $10$ chains
with the required level of acceptance rate ($\ge 20\%$)
are obtained.





Figure~\ref{exact_vs_approx_time_barplots} shows a
barplot of runtimes (in hours)
for the four scenarios under consideration.
Purple bars represent runtimes
for the $10$ retained chains per scenario,
whereas gray bars indicate runtimes
for the chains that have been discarded
due to unmet acceptance rate requirements.
As seen from a comparison between purple bars,
approximate NBG has shorter runtime
(for retained chains of same length) than exact NBG,
which is explained by the fact that
minibatching uses a subset of the training set
at each approximate NBG iteration.
A comparison between gray bars in scenarios $3$ and $4$
demonstrates that exact NBG runtimes for discarded chains
increase with increasing acceptance rate lower bounds.
By observing Figures~\ref{exact_vs_approx_pred_acc_boxplots}
and~\ref{exact_vs_approx_time_barplots} jointly,
it is pointed out that
predictive accuracy improvements of exact NBG
(arising from higher acceptance rate lower bounds)
come at higher computational costs.

\begin{obs}
Exact MCMC algorithms
based on the Metropolis-Hastings acceptance mechanism
are not feasible
for feedforward neural networks
due to vanishing acceptance rates and
high computational cost.
Splitting the parameter space into smaller blocks
recovers higher acceptance rates,
and minibatch MCMC sampling reduces
the computational cost per sampling step.
With relatively small penalty in predictive accuracy,
minibatch blocked Gibbs sampling makes it possible
to traverse the parameter space with reduced computational cost.
Being able to shift from no mixing of exact MCMC
to slow mixing of approximate MCMC
yields gains in predictive accuracy.
\end{obs}

\subsection{Depth and acceptance rate}
\label{depth_and_acceptance}

Figure~\ref{acc_rate_and_depth} displays
mean acceptance rates
across $m=10$ chains realized via minibatch NBG
upon fitting $\mbox{MLP}(2, 2, 2, 2, 2, 2, 2, 1)$ to noisy XOR.
In particular,
Figure~\ref{acc_rate_per_node}
shows the mean acceptance rate for each node
in the six hidden layers and one output layer of
$\mbox{MLP}(2, 2, 2, 2, 2, 2, 2, 1)$,
while Figure~\ref{acc_rate_per_layer}
shows the mean acceptance rate for each
of these seven (six hidden and one output) layers.
A batch size of $100$ is used for minibatch NBG.
The same set of $10$ chains have been used
in Figures~\ref{exact_vs_approx_plots}
and~\ref{acc_rate_and_depth}.

\begin{figure}[t]
	\begin{subfigure}{.491\textwidth}
		\centering
		\includegraphics[width=1\linewidth]{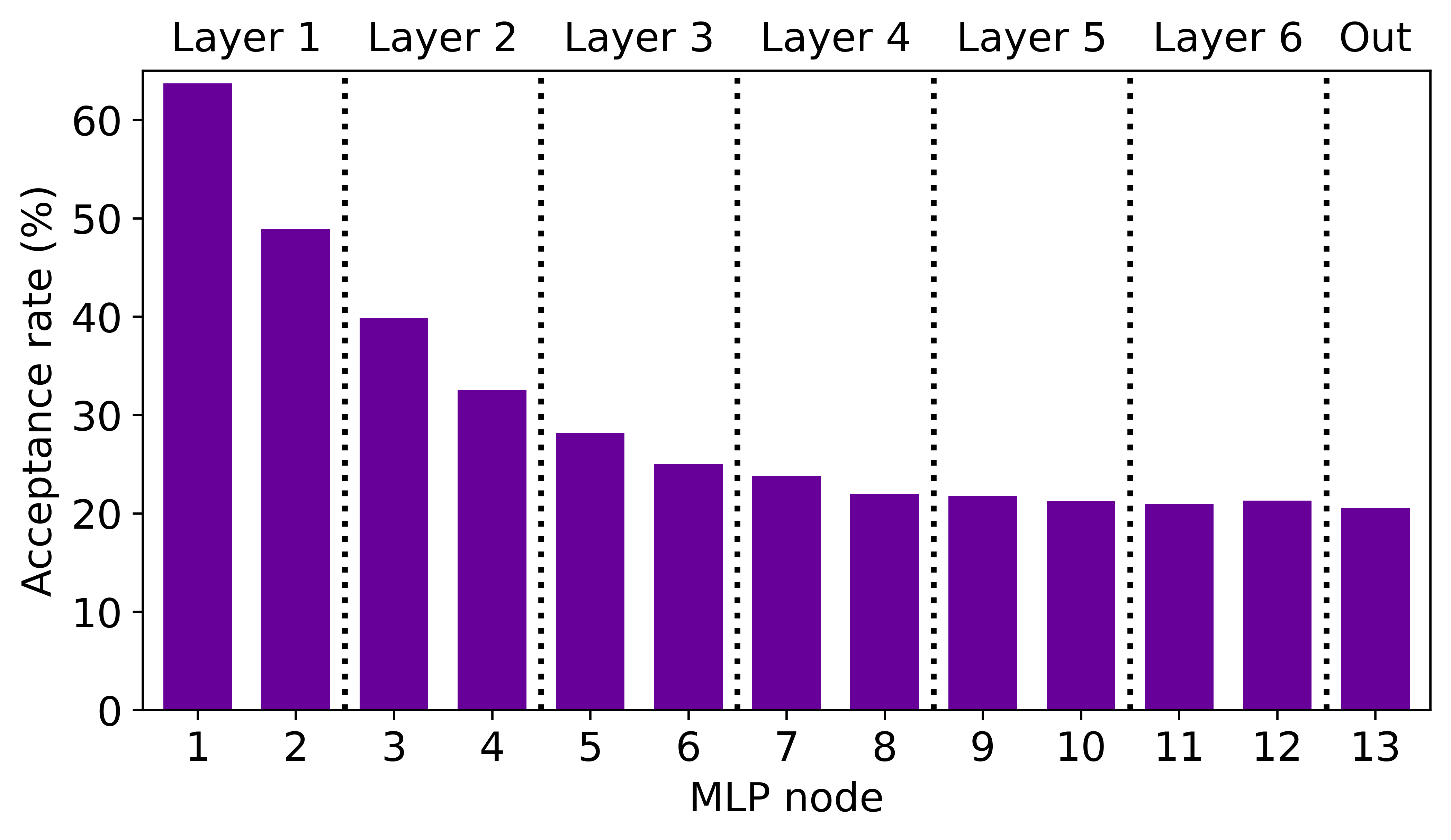}
		\caption{Mean acceptance rate per node.}
		\label{acc_rate_per_node}
	\end{subfigure}\\
	\begin{subfigure}{.491\textwidth}
		\centering
		\includegraphics[width=1\linewidth]{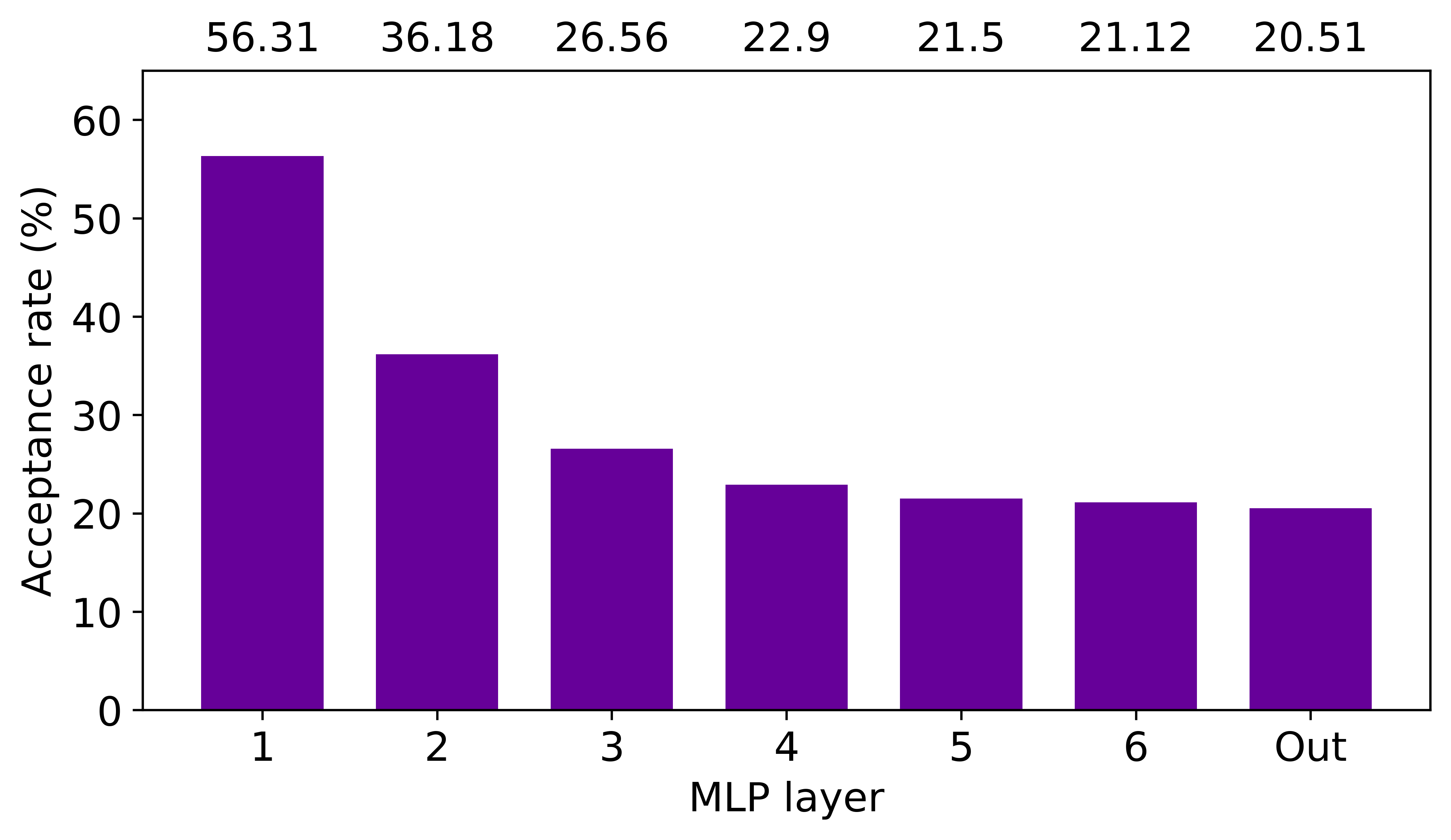}
		\caption{Mean acceptance rate per layer.}
		\label{acc_rate_per_layer}
	\end{subfigure}
	\caption{Mean acceptance rates (per node and per layer)
		across $10$ chains realized via minibatch NBG sampling
		of the $\mbox{MLP}(2, 2, 2, 2, 2, 2, 2, 1)$ parameters.
		The MLP is fitted to noisy XOR.
		A batch size of $100$ is used.}
	\label{acc_rate_and_depth}
\end{figure}

Figures~\ref{exact_vs_approx_acc_rate_boxplots}
and~\ref{acc_rate_per_node}
provide alternative views of node-specific acceptance rates.
The former figure represents such information via boxplots and medians,
whereas the latter makes use of a barplot of associated means.

Figure~\ref{acc_rate_and_depth}
demonstrates that
if the proposal variance is the same
for all parameter blocks across layers,
then the acceptance rate reduces with depth.
For instance,
it can be seen in Figure~\ref{acc_rate_per_layer}
that the acceptance rates
for hidden layers $1$, $2$ and $3$ are
$56.31\%$, $36.18\%$ and $26.56\%$, respectively.

Using a common proposal variance
for all parameter blocks across layers
generates disparities in acceptance rates,
with higher rates in shallower layers
and lower rates in deeper layers.
These disparities become more pronounced
with big data or with high parameter dimensions.
For example, sampling
$\mbox{MLP}(784, 10, 10, 10, 10)$
parameters
with the same proposal variance
in all parameter blocks
is not feasible
in the case of MNIST or FMNIST;
the acceptance rates are high in the first hidden layer
and drop near zero in the output layer.
FNBG sampling enables to
reduce the proposal variance for deeper layers,
thus avoiding vanishing acceptance rates
with increasing depth.

Tables~\ref{mnist_tuning_layers} and~\ref{fmnist_tuning_layers}
of~\ref{app:tuning} exemplify
empirically tuned proposal variances
for minibatch FNBG sampling of
$\mbox{MLP}(784, 10, 10, 10, 10)$
parameters
in the respective cases of MNIST and FMNIST.
Batch sizes of
$600,~1800,~3000$ and $4200$
are employed,
corresponding to
$1\%,~3\%,~5\%$ and $7\%$ of the
MNIST and FMNIST training sets.
For each of these four batch sizes
and for each training set,
the proposal variance per layer
\tp{is reduced during pre-training
until the acceptance rate of the layer is not prohibitively low,
and subsequently
the proposal variance tuned via pre-training
is used for computing
the acceptance rate of the corresponding layer from a chain realization.}
Tables~\ref{mnist_tuning_layers} and~\ref{fmnist_tuning_layers}
demonstrate that
if proposal variances are reduced in deeper layers,
then acceptance rates do not vanish with depth.
For increasing batch size,
acceptance rates drop across all layers,
as expected when shifting from approximate towards exact MCMC.

As part of Table~\ref{mnist_tuning_layers},
a chain is simulated
upon fitting
$\mbox{MLP}(784, 10, 10, 10, 10)$
to the MNIST training set
via minibatch FNBG sampling
with a batch size of $3000$.
Figure~\ref{traceplots},
which comprises a grid of $4\times 2=8$ traceplots,
is produced from that chain.
Each row of Figure~\ref{traceplots}
is related to one of the $8180$ parameters of
$\mbox{MLP}(784, 10, 10, 10, 10)$.
More specifically,
the first, second, third and fourth row
correspond to
parameter $\theta_{1005}$ in hidden layer $1$,
parameter $\theta_{7872}$ in hidden layer $2$,
parameter $\theta_{8008}$ in hidden layer $3$ and
parameter $\theta_{8107}$ in the output layer.
A pair of traceplots per parameter is shown in each row;
the right traceplot is more zoomed out
than the left one.
All traceplots in the right column
share a common range of $[-8, 8]$ in their vertical axes.

\begin{figure*}[t]
	\centering
	\includegraphics[width=1\linewidth]{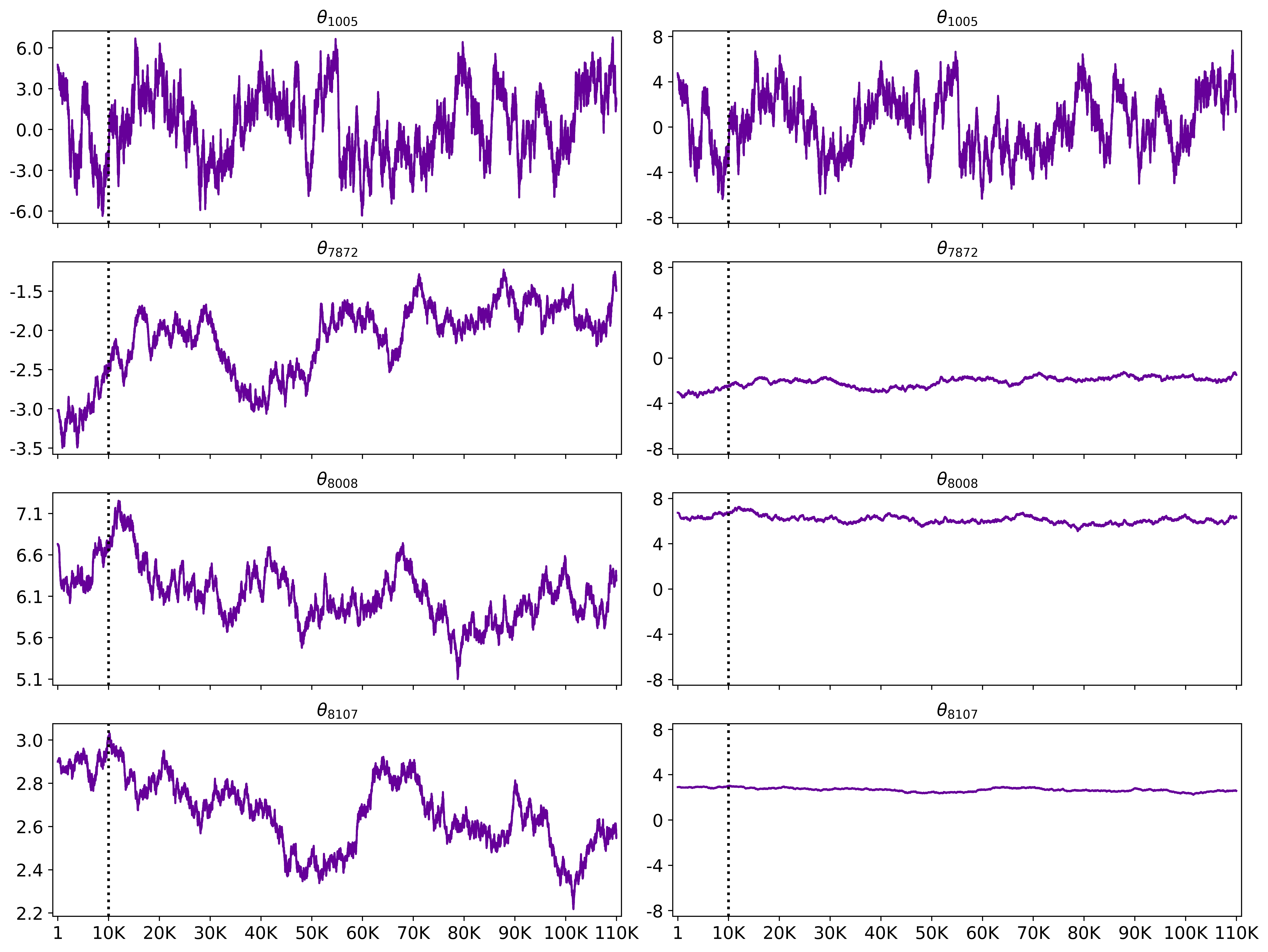}
	\caption{Markov chain traceplots of
		four parameter coordinates
		of $\mbox{MLP}(784, 10, 10, 10, 10)$, which is fitted to MNIST
		via minibatch FNBG sampling with a batch size of $3000$.
		Each row displays two traceplots of the same chain for a single parameter;
		the traceplot on the right is more zoomed-out than the one on the left.
		The traceplots of the right column share a common range on the vertical axes.
		Vertical dotted lines indicate the end of burnin.}
	\label{traceplots}
\end{figure*}

It is observed that the zoomed-in traceplots
(left column of Figure~\ref{traceplots})
do not exhibit entrapment in local modes
irrespective of network depth,
agreeing with the non-vanishing acceptance rates
of Table~\ref{mnist_tuning_layers}.
Furthermore, it is seen from the zoomed-out traceplots
(right column of Figure~\ref{traceplots})
that chain scales decrease in deeper layers.
For example,
the right traceplot of parameter $\theta_{8107}$ (output layer)
has non-visible fluctuations under a y-axis range of $[-8, 8]$,
whereas the right traceplot
of parameter $\theta_{1005}$ (first hidden layer)
fluctuates more widely under the same y-axis range.

Figure~\ref{traceplots} suggests that
chains of parameters in shallower layers perform more exploration,
while chains of parameters in deeper layers carry out more exploitation.
This way, chain scales collapse towards point estimates
for increasing network depth.

\subsection{Batch size and log-likelihood}
\label{sec:batch_logl}

For each batch size shown
in
Figure~\ref{mnist_log_lik_boxplots},
the likelihood function of Equation~\eqref{mc_mlp_lik}
is evaluated on $10$ batch samples,
which are drawn from the MNIST training set.
A boxplot is then generated
from the $10$ log-likelihood values
and it is displayed in Figure~\ref{mnist_log_lik_boxplots}.
The log-likelihood function is normalized by batch size
in order to obtain visually comparable boxplots
across different batch sizes.
In \texttt{PyTorch},
the normalized log-likelihood is computed
via the \texttt{CrossEntropyLoss} class
initialized with
\texttt{reduction=\textquotesingle mean\textquotesingle}.
In each boxplot, the blue line and yellow point
correspond to the median and mean
of the $10$ associated log-likelihood values.
The horizontal gray line represents
the log-likelihood value
based on the whole MNIST training set.
Figure~\ref{fmnist_log_lik_boxplots}
is generated using the FMNIST training set,
following an analogous setup.

\begin{figure}[t]
	\begin{subfigure}{.491\textwidth}
		\centering
		\includegraphics[width=1\linewidth]{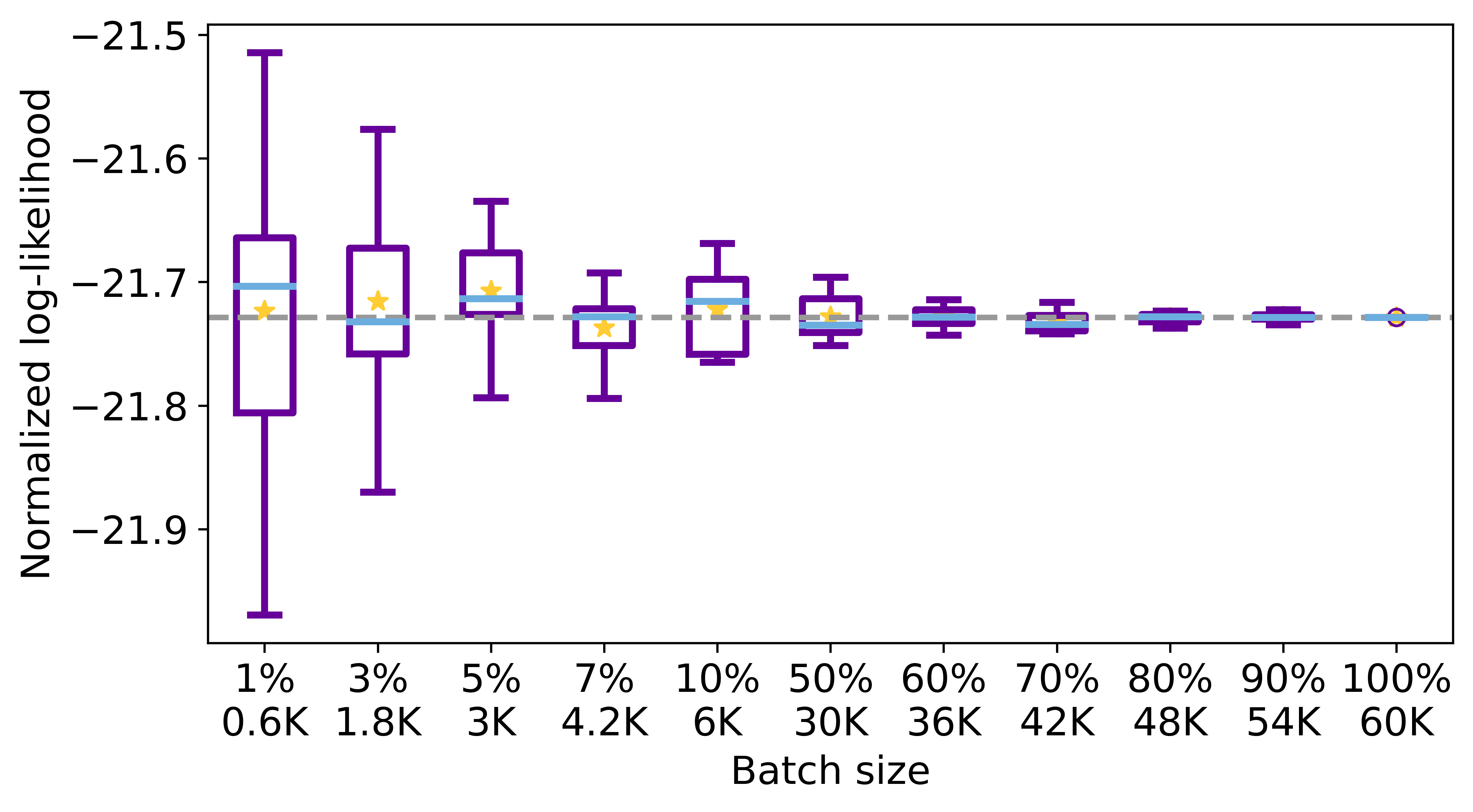}
		\caption{Log-likelihood value boxplots for MNIST.}
		\label{mnist_log_lik_boxplots}
	\end{subfigure}\\
	\begin{subfigure}{.491\textwidth}
		\centering
		\includegraphics[width=1\linewidth]{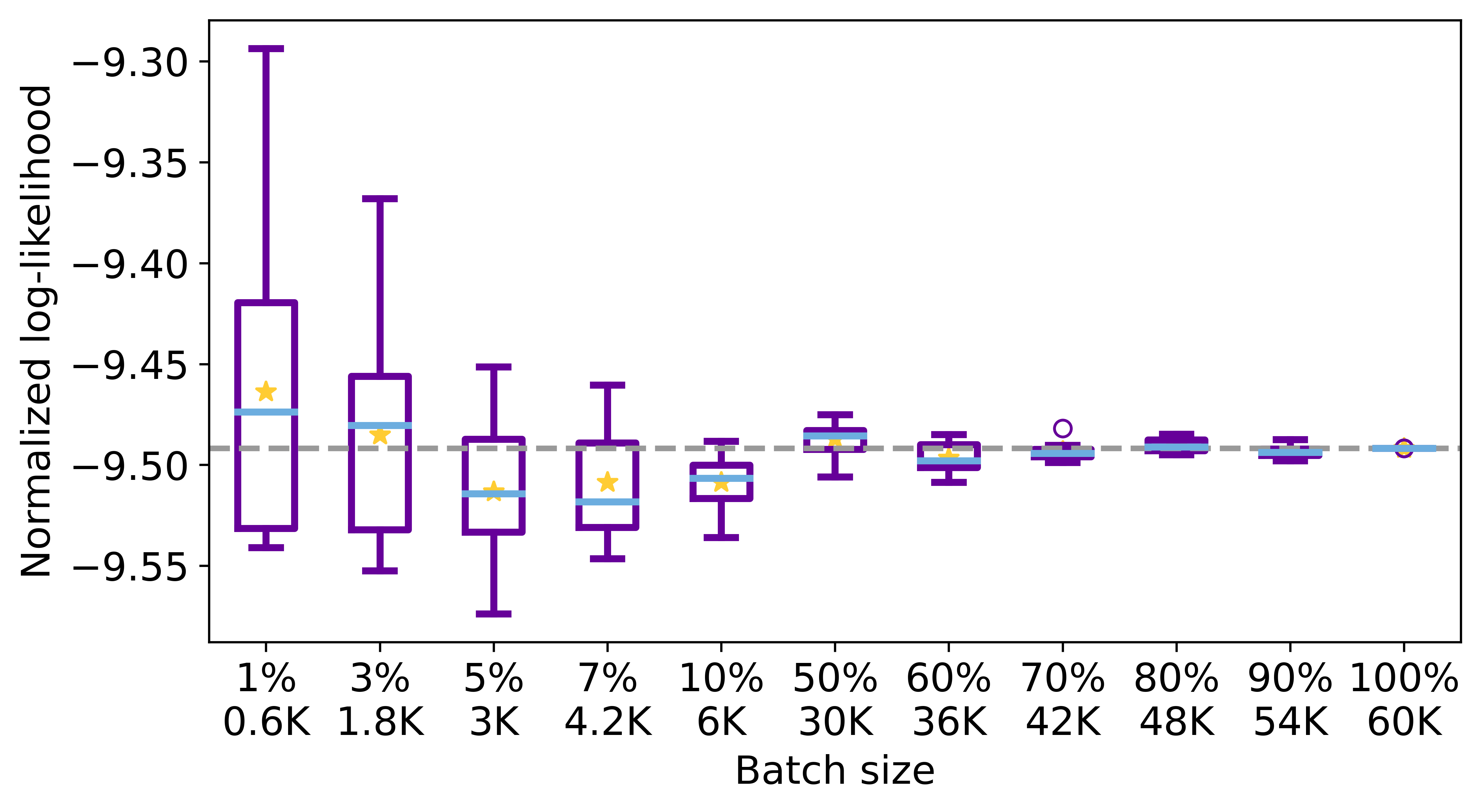}
		\caption{Log-likelihood value boxplots for FMNIST.}
		\label{fmnist_log_lik_boxplots}
	\end{subfigure}
	\caption{Boxplots of normalized log-likelihood values
		for MNIST and FMNIST.
		Each boxplot summarizes normalized log-likelihood values
		of $10$ batch samples for a given batch size.
		To normalize, each log-likelihood value is divided by batch size.
		Blue lines and yellow points correspond to medians and means.
		Horizontal gray lines represent exact log-likelihood values
		for batch size equal to training sample size.}
	\label{log_lik_boxplots}
\end{figure}

Figures~\ref{mnist_log_lik_boxplots} and~\ref{fmnist_log_lik_boxplots}
demonstrate that
log-likelihood values are increasingly volatile
for decreasing batch size.
Furthermore, the volatility of log-likelihood values vanishes
as the batch size gets close to the training sample size.
So, Figure~\ref{log_lik_boxplots} confirms visually
that the approximate likelihood tends to the exact likelihood
for increasing batch size.
Thus, the batch size in FNBG sampling is preferred
to be as large as possible,
up to the point that (finer) block acceptance rates
do not become prohibitively low.

\subsection{Depth and predictions}

Figure~\ref{shallow_vs_deep_boxplots}
explores how network depth affects predictive accuracy in approximate MCMC.
Shallower
$\mbox{MLP}(2, 2, 1)$,
consisting of one hidden layer,
and deeper
$\mbox{MLP}(2, 2, 2, 2, 2, 2, 2, 1)$,
consisting of six hidden layers,
are fitted to the noisy XOR training set using minibatch NBG
with a batch size of $100$ and a proposal variance of $0.04$;
$m=10$ chains are realized for each of the two MLPs.
Subsequently, the predictive accuracy per chain is evaluated
on the noisy XOR test set.
One boxplot is generated for each set of $10$ chains,
as shown in Figure~\ref{shallow_vs_deep_boxplots}.
Blue lines represent medians.

\begin{figure}[t]
	\centering
	\includegraphics[width=1\linewidth]{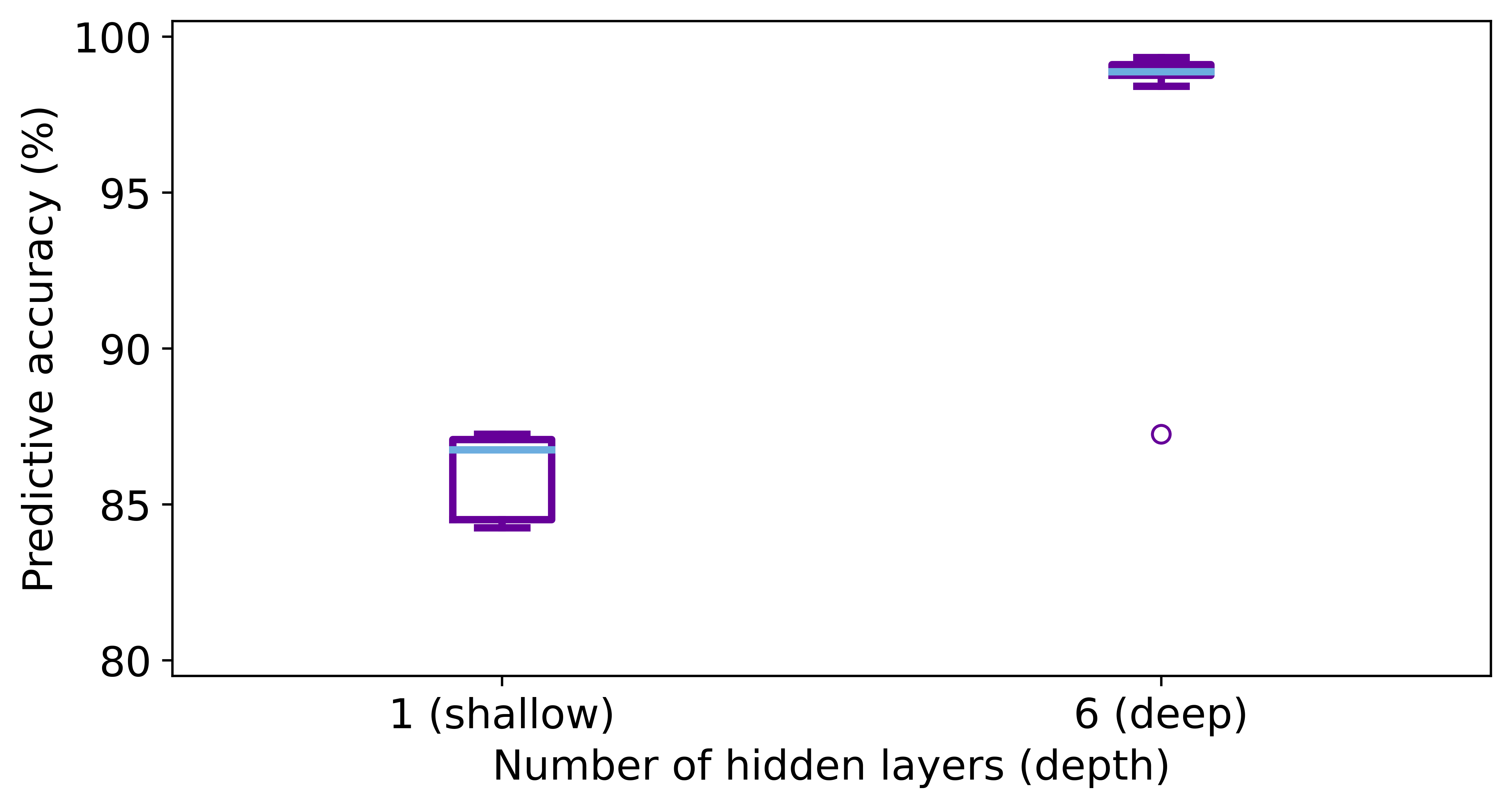}
	\caption{A comparison between a shallower and a deeper MLP architecture.
		Each of
		$\mbox{MLP}(2, 2, 1)$
		and
		$\mbox{MLP}(2, 2, 2, 2, 2, 2, 2, 1)$
		is fitted to noisy XOR
		via minibatch NBG sampling with a batch size of $100$.
		Predictive accuracy boxplots are generated from $10$ chains per MLP.
		Blue lines indicate medians.
	}
	\label{shallow_vs_deep_boxplots}
\end{figure}

The same $10$ chains are used
to generate relevant plots in
Figures~\ref{exact_vs_approx_pred_acc_boxplots},
~\ref{acc_rate_and_depth}
and~\ref{shallow_vs_deep_boxplots}.
In particular,
the leftmost boxplot
in Figure~\ref{exact_vs_approx_pred_acc_boxplots}
and right boxplot
in Figure~\ref{shallow_vs_deep_boxplots}
stem from the same $10$ chains and are thus identical.
Figure~\ref{acc_rate_and_depth}
shows mean acceptance rates per node and per layer
across the $10$ chains that also yield
the right boxplot of predictive accuracies
in Figure~\ref{shallow_vs_deep_boxplots}.

$\mbox{MLP}(2, 2, 1)$
and
$\mbox{MLP}(2, 2, 2, 2, 2, 2, 2, 1)$ have
respective predictive accuracy medians of
$86.75\%$ and $98.88\%$
as blue lines indicate in Figure~\ref{shallow_vs_deep_boxplots},
so predictive accuracy increases with increasing depth.
Moreover, the interquartile ranges 
of Figure~\ref{shallow_vs_deep_boxplots}
demonstrate that a deeper architecture yields less volatile,
and in that sense more stable, predictive accuracy.
As an overall empirical observation,
increasing the network depth in approximate MCMC
seems to produce higher and less volatile predictive accuracy.



\begin{obs}
Increasing the depth of a feedforward neural network
increases the predictive accuracy but
reduces the acceptance rates for blocks in deeper layers.
Reducing the proposal variance in deeper layers
helps counter the reduction of acceptance rates.
Increasing the network width in initial layers
does not have a negative impact on acceptance rates,
in contrast to the negative impact of
increasing depth on acceptance rates.
\end{obs}

\subsection{Batch size and predictions}
\label{subs:bs_pred}

This subsection assesses empirically
the effect of batch size on predictive accuracy
in approximate MCMC.
To this end,
$\mbox{MLP}(784, 10, 10, 10, 10)$
is fitted to the MNIST and FMNIST training sets
using minibatch FNBG sampling with batch sizes of
$600$, $1800$, $3000$ and $4200$,
which correspond to
$1\%$, $3\%$, $5\%$ and $7\%$
of each training sample size.
One chain is realized
per combination of training set and batch size.
Table~\ref{tab_batch_size_acc}
reports the predictive accuracy
for each chain.

\begin{table}[t]
	\centering
	\caption{Predictive accuracies
		obtained by fitting
		$\mbox{MLP}(784, 10, 10, 10, 10)$
		to MNIST and to FMNIST
		via minibatch FNBG sampling
		with different batch sizes.}
	\label{tab_batch_size_acc}
	\begin{tabular}{lrrrr}
		\hline
		\multicolumn{1}{c}{\multirow{3}{*}{Dataset}} &
		\multicolumn{4}{c}{Batch size} \\ \cline{2-5}
		&
		\multicolumn{1}{c}{1\%}  &
		\multicolumn{1}{c}{3\%}  &
		\multicolumn{1}{c}{5\%}  &
		\multicolumn{1}{c}{7\%} \\
		&
		\multicolumn{1}{c}{0.6K}  &
		\multicolumn{1}{c}{1.8K}  &
		\multicolumn{1}{c}{3K}  &
		\multicolumn{1}{c}{4.2K} \\
		\hline
		MNIST  & 85.99 & 89.01 & 90.75 & 90.43 \\
		FMNIST & 71.50 & 80.07 & 80.89 & 79.17 \\
		\hline
	\end{tabular}
\end{table}

The same chains are used
to compute predictive accuracies
in Table~\ref{tab_batch_size_acc}
as well as acceptance rates
in Tables~\ref{mnist_tuning_layers}
and~\ref{fmnist_tuning_layers}
of~\ref{app:tuning}.
The chain that yields
the predictive accuracy
for MNIST and for a batch size of $3000$
(first row and third column of Table~\ref{tab_batch_size_acc})
is partly visualized by traceplots in Figure~\ref{traceplots}.

According to Table~\ref{tab_batch_size_acc},
the highest accuracy of $90.75\%$ for MNIST
and of $80.89\%$ for FMNIST
are attained by employing a batch size of $3000$.
Overall, predictive accuracy increases as batch size increases.
However, predictive accuracy decreases
when batch size increases from $3000$ to $4200$;
this is explained by the fact that
a batch size of $4200$ is too large,
in the sense that it reduces acceptance rates
(see Tables~\ref{mnist_tuning_layers}
and~\ref{fmnist_tuning_layers}).
So, as pointed out in Subsection~\ref{sec:batch_logl},
a tuning guideline is to increase the batch size
up to the point that no substantial reduction
in finer block acceptance rates occurs.

An attained predictive accuracy of $90.75\%$ on MNIST
demonstrates that non-convergent chains (simulated
via minibatch FNBG) learn from data,
since data-agnostic guessing based on pure chance
has a predictive accuracy of $10\%$.
While stochastic optimization algorithms for deep learning
achieve predictive accuracies higher than $90.75\%$ on MNIST,
the goal of this work
has not been to construct an approximate MCMC algorithm
that outperforms stochastic optimization
on the predictive front.
The main objective has been
to demonstrate that approximate MCMC for neural networks
learns from data and
to uncover associated sampling characteristics,
such as diminishing chain ranges (Figure~\ref{traceplots})
and diminishing acceptance rates (Tables~\ref{mnist_tuning_layers}
and~\ref{fmnist_tuning_layers})
for increasing network depth.
Similar predictive accuracies in the vicinity of $90\%$
using Hamiltonian Monte Carlo for deep learning
have been reported
in the literature~\citep{wenzel2020, izmailov2021}.
Nonetheless, this body of relevant work
relies on chain lengths
one or two orders of magnitude shorter;
\tp{for instance, \cite{izmailov2021}
have run up to $900$ iterations per chain realization.} 
The present paper proposes
to circumvent vanishing acceptance rates
by grouping neural network parameters into smaller blocks,
thus enabling the generation of lengthier chains.

\begin{obs}
Increasing the batch size
in minibatch MCMC sampling
of feedforward neural network parameters
increases the predictive accuracy.
This observation is anticipated,
in the sense that minibatch MCMC
becomes exact MCMC
when the batch size is
equal to the training sample size.
However, the batch size can be increased
up to the point that no substantial reduction
in acceptance rates occurs.
\end{obs}



\subsection{Chain length and predictions}

It is reminded that $110000$ iterations
are run per chain in the experiments herein,
of which the first $10000$ are discarded as burnin.
The last $v=10000$ (out of the remaining $100000$) iterations
are used for making predictions
via Bayesian marginalization
based on Equation~\eqref{pred_posterior_approx}.
Only $10000$ iterations are utilized
in Equation~\eqref{pred_posterior_approx}
to cap the computational cost for predictions.

There exists a tractable solution to Bayesian marginalization,
since the approximate posterior predictive pmf
of Equation~\eqref{pred_posterior_approx}
can be computed in parallel
both in terms of Monte Carlo iterations and of test points.
The implementation of such a parallel solution is deferred to future work.

In the meantime, it is examined here
how chain length affects predictive accuracy.
Along these lines, predictive accuracies are computed
from the last $1000$, $10000$, $20000$ and $30000$ iterations
of the chain realized via minibatch FNBG
with a batch size of $3000$ for each of MNIST and FMNIST
(see Table~\ref{tab_chain_len_acc}).
The last $10000$ and all $100000$ post-burnin iterations
of the same chain
generate predictive accuracies
in Table~\ref{tab_batch_size_acc} and
acceptance rates in Tables~\ref{mnist_tuning_layers}
and~\ref{fmnist_tuning_layers}, respectively.

\begin{table}[t]
	\centering
	\caption{Predictive accuracies
		obtained from different chain lengths.
		$\mbox{MLP}(784, 10, 10, 10, 10)$
		is fitted to MNIST and to FMNIST
		via minibatch FNBG sampling
		with a batch size of $3000$.
		One chain is realized per dataset.
		Subsequently, predictions are made
		via Bayesian marginalization
		using chunks of different length
		from the end of the realized chains.}
	\label{tab_chain_len_acc}
	\begin{tabular}{lrrrr}
		\hline
		\multicolumn{1}{c}{\multirow{2}{*}{Dataset}} &
		\multicolumn{4}{c}{Chain length} \\ \cline{2-5}
		&
		\multicolumn{1}{c}{1K}  &
		\multicolumn{1}{c}{10K} &
		\multicolumn{1}{c}{20K} &
		\multicolumn{1}{c}{30K} \\
		\hline
		MNIST  & 88.31 & 90.75 & 91.12 & 91.20 \\
		FMNIST & 78.93 & 80.89 & 81.36 & 81.53 \\
		\hline
	\end{tabular}
\end{table}

Table~\ref{tab_chain_len_acc}
demonstrates that predictive accuracy increases
(both for MNIST and FMNIST)
as chain length increases.
So, as a chain traverses the parameter space of a neural network,
information of predictive importance accrues
despite the lack of convergence.
It can also be seen from
Table~\ref{tab_chain_len_acc}
that the rate of improvement in predictive accuracy
slows down for increasing chain length.


\begin{obs}
Despite the lack of convergence and the slow mixing,
increasing the number of approximate MCMC iterations
upon sampling from the parameter space
of a feedforward neural network
increases the predictive accuracy.
The rate of improvement in predictive accuracy
slows down for increasing chain length.
\end{obs}

\subsection{Augmentation and predictions}
\label{subs:data_aug_pred}

To assess the effect of data augmentation on predictive accuracy,
three image transformations are performed on the MNIST and FMNIST training sets,
namely rotations by angle, blurring, and colour inversions.
Images are rotated by angles randomly selected between $-30$ and $30$ degrees.
Each image is blurred with probability $0.9$.
Blur is randomly generated from a Gaussian kernel of size $9\times 9$.
The standard deviation of the kernel is randomly selected between $1$ and $1.5$.
Each image is colour-inverted with probability $0.5$.
Figure~\ref{data_augmentation} in~\ref{app:augmentation}
displays examples of MNIST and FMNIST training images
that have been rotated, blurred or colour-inverted
according to the described transformations.

Each of the three transformations is applied
to the whole MNIST and FMNIST training sets.
Subsequently,
$\mbox{MLP}(784, 10, 10, 10, 10)$
is fitted to each transformed training set via minibatch FNBG
with a batch size of $3000$ and
with proposal variances specified
in Tables~\ref{mnist_tuning_layers}
and~\ref{fmnist_tuning_layers}.
One chain is simulated per transformed training set.
Predictive accuracies are computed
on the corresponding untransformed MNIST and FMNIST test sets
and are reported in Table~\ref{tab:augmentation_acc}.
Moreover, predictive accuracies
based on the untransformed MNIST and FMNIST training sets
are available in the first column of
Table~\ref{tab:augmentation_acc},
as previously reported in Table~\ref{tab_batch_size_acc}.

\begin{table}[t]
	\centering
	\caption{Predictive accuracies
		obtained from different data augmentation schemes.
		$\mbox{MLP}(784, 10, 10, 10, 10)$
		is fitted to each of the
		augmented MNIST and FMNIST training sets
		via minibatch FNBG sampling
		with a batch size of $3000$.
		Predictive accuracies are computed on
		the corresponding non-augmented test sets.
		The first column reports predictive accuracies
		based on the non-augmented MNIST and FMNIST training sets.}
	\label{tab:augmentation_acc}
	\begin{tabular}{lrrrr}
		\hline
		\multicolumn{1}{c}{\multirow{2}{*}{Dataset}} &
		\multicolumn{4}{c}{Transform} \\ \cline{2-5}
		&
		\multicolumn{1}{c}{None} &
		\multicolumn{1}{c}{Rotation} &
		\multicolumn{1}{c}{Blur} &
		\multicolumn{1}{c}{Inversion} \\
		\hline
		MNIST  & 90.75 & 86.19 & 85.66 & 36.87 \\
		FMNIST & 80.89 &  6.62 &  7.46 &  8.61 \\
		\hline
	\end{tabular}
\end{table}

According to Table~\ref{tab:augmentation_acc},
if data augmentation is performed,
then predictive accuracy deteriorates drastically.
Notably, data augmentation
has catastrophic predictive consequences for FMNIST.
These empirical findings agree with the
`dirty likelihood hypothesis'
of~\cite{wenzel2020},
according to which data augmentation
violates the likelihood principle.

\begin{obs}
Approximate MCMC sampling
of feedforward neural network parameters
in the presence of augmented data
remains an open problem.
Data augmentation violates the likelihood principle
and consequently reduces drastically
the predictive accuracy.
\end{obs}

\subsection{Uncertainty quantification}


Approximate MCMC enables
predictive uncertainty quantification (UQ)
via Bayesian marginalization.
Such a principled approach to UQ
constitutes an advantage of approximate MCMC
over stochastic optimization in deep learning.
This subsection showcases
how predictive uncertainty
is quantified for neural networks
via minibatch FNBG sampling.

Recall that one chain has been simulated
for each of MNIST and FMNIST
to compute the predictive accuracies
of column $3$ in Table~\ref{tab_batch_size_acc}
(see Subsection~\ref{subs:bs_pred}).
Those chains
are used to estimate posterior predictive probabilities
for some images in the corresponding test sets,
as shown in Figure~\ref{uq}.
All test images in Figure~\ref{uq}
have been correctly classified
via Bayesian marginalization.

\begin{figure}[t]
	\begin{subfigure}{.491\textwidth}
		\centering
		\includegraphics[width=1\linewidth]{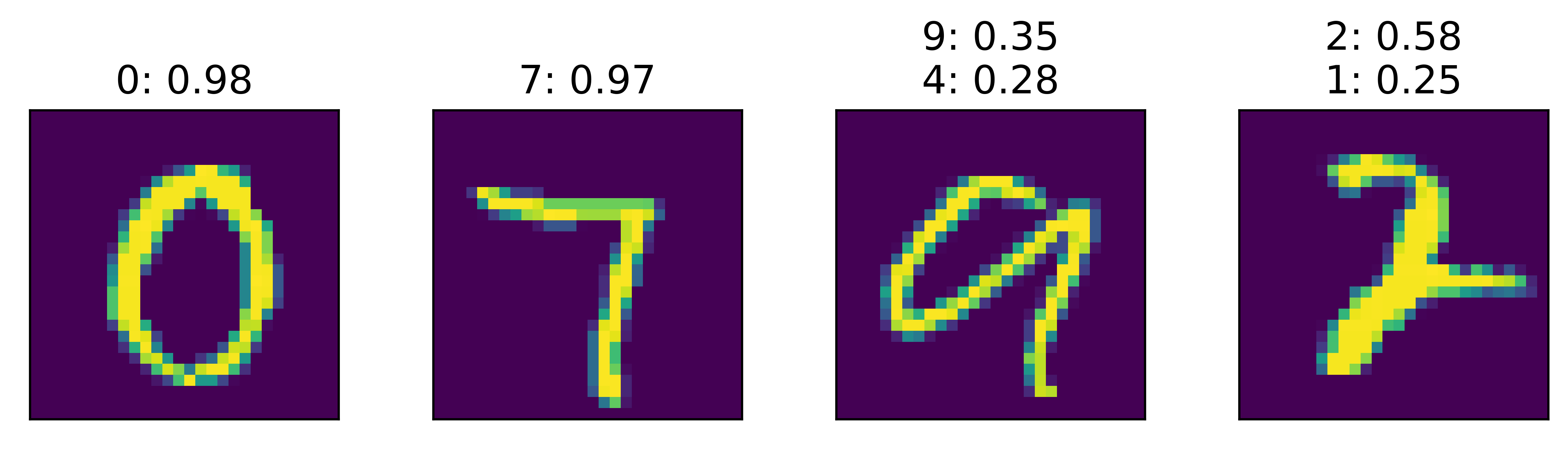}
		\caption{Posterior predictive probabilities for MNIST.}
		\label{mnist_uq}
	\end{subfigure}\\
	\begin{subfigure}{.491\textwidth}
		\centering
		\includegraphics[width=1\linewidth]{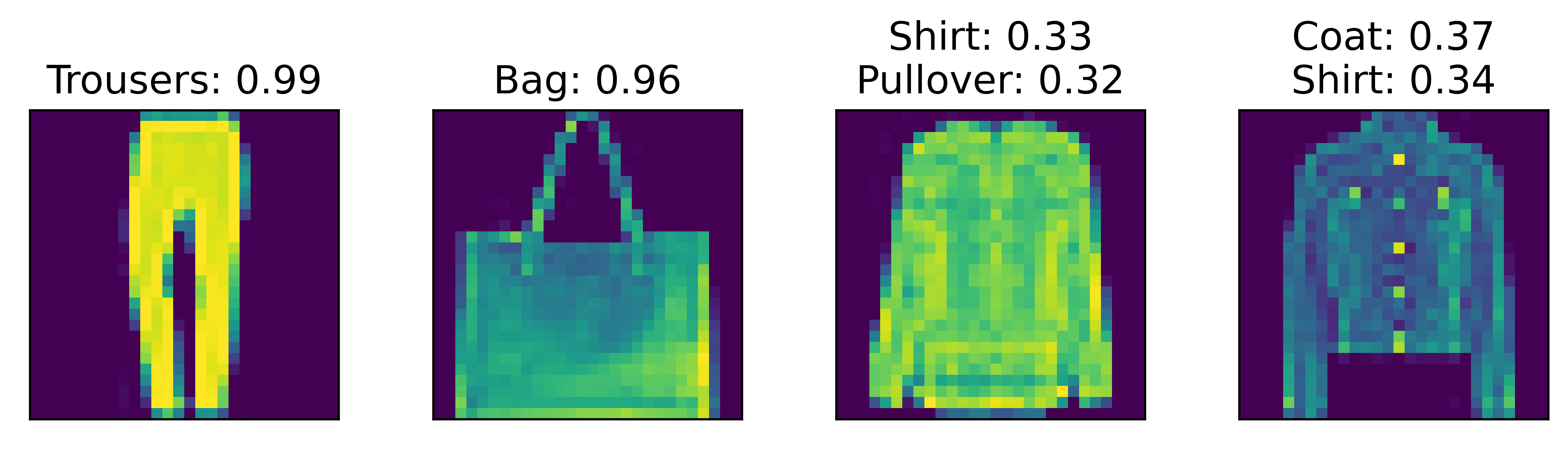}
		\caption{Posterior predictive probabilities for FMNIST.}
		\label{fmnist_uq}
	\end{subfigure}
	\caption{Demonstration of UQ
		for some correctly classified MNIST and FMNIST test images.
		The highest posterior predictive probability is displayed
		for each image associated with low uncertainty,
		whereas the two highest posterior predictive probabilities
		are displayed for each image
		associated with high uncertainty.}
	\label{uq}
\end{figure}

The first and second MNIST test images
in Figure~\ref{mnist_uq}
show numbers $0$ and $7$,
with corresponding posterior predictive probabilities
$0.98$ and $0.97$
that indicate near-certainty
about the classification outcomes.
The third MNIST test image
in Figure~\ref{mnist_uq}
shows number $9$.
Attempting to classify this image by eye
casts doubt as to whether the number in the image
is $9$ or $4$.
While Bayesian marginalization correctly
classifies the number as $9$,
the posterior predictive probability
$\hat{p}(y=9 \vert x,\mathcal{D}_{1:s})=0.35$
is relatively low,
indicating uncertainty in the prediction.
Moreover, the second highest
posterior predictive probability
$\hat{p}(y=4 \vert x,\mathcal{D}_{1:s})=0.28$
identifies number $4$ as a probable alternative,
in agreement with human perception.
All in all, posterior predictive probabilities
and human understanding are aligned
in terms of perceived predictive uncertainties
and in terms of plausible classification outcomes.
Image $4$ is aligned with image $3$ of
Figure~\ref{mnist_uq} regarding UQ conclusions.

Figure~\ref{fmnist_uq},
which entails FMNIST test images,
is analogous to Figure~\ref{mnist_uq}
from a UQ point of view.
In Figure~\ref{fmnist_uq},
FMNIST test images $1$ and $2$
show trousers and a bag,
with corresponding posterior predictive probabilities
$0.99$ and $0.96$
that indicate near-certainty about the classification outcomes.
The third FMNIST test image
of Figure~\ref{fmnist_uq} shows a shirt.
It is not visually clear whether this image depicts
a shirt or a pullover.
While Bayesian marginalization correctly identifies
the object as a shirt,
the posterior predictive probabilities
$\hat{p}(y=\mbox{shirt} \vert x,\mathcal{D}_{1:s})=0.33$
and
$\hat{p}(y=\mbox{pullover} \vert x,\mathcal{D}_{1:s})=0.32$
capture human uncertainty and
identify the two most plausible classification outcomes.
Image $4$ is analogous to image $3$
of Figure~\ref{fmnist_uq}
in terms of UQ conclusions.

\begin{obs}
A non-convergent chain realization
via approximate MCMC sampling
of feedforward neural network parameters
can help with the assessment of predictive uncertainty meaningfully,
that is in agreement with human insights.
\end{obs}

\section{Future work}
\label{discussion}


Several future research directions emerge from this paper;
\tp{three} software engineering extensions are planned,
three methodological developments are proposed,
and one theoretical question is posed.

To start with possible software engineering work,
Bayesian marginalization can be parallelized
across test points and across FNBG iterations per test point.
Additionally, an adaptive version of FNBG sampling
can be implemented based on existing Gibbs sampling methods
for proposal variance tuning~\tp{\citep{andrieu2008}},
thus automating tuning and
reducing tuning computational requirements.
\tp{Moreover, FNBG sampling can be implemented with a subsampling mechanism
that sets the batch size adaptively~\citep{bardenet2014}.}

In terms of methodological developments,
alternative ways of grouping parameters in FNBG sampling
may be considered.
For example, parameters may be grouped
according to their covariance structure,
as estimated from pilot FNBG runs.
Furthermore, functional priors proposed by~\cite{tran2022}
or adaptations of them may be utilized in conjunction with FNBG.
Moreover, FNBG sampling may be developed
for neural network architectures other than MLPs.
To this end, DAG representations
of other neural network architectures
will be devised and
fine parameter blocks will be identified from the DAGs.

A theoretical question of interest is
how to construct lower bounds of predictive accuracy
for minibatch FNBG (and for minibatch MCMC more generally)
as a function of the distance between
the exact and approximate parameter posterior density.
It has been observed empirically that
minibatch FNBG has predictive capacity,
yet theoretical guarantees for predictive accuracy
have not been established.

The proposed sampling approach and future developments
face two main limitations.
Firstly, it remains an open question how to sample
neural network parameters
given augmented training data,
as previously pointed out by the `dirty likelihood hypothesis' of
~\cite{wenzel2020}.
Secondly,
as the depth of a feedforward neural network increases,
the proposal variance of FNBG is reduced for deeper layers.
Thus, the proposal variance for deeper layers
may be set to a value too close to zero
from a practical point of view.


\section*{Software and data}

The FNBG sampler
for MLPs
has been implemented under the \texttt{eeyore} package
using \texttt{Python} and \texttt{PyTorch}.
\texttt{eeyore} is available at
\url{https://github.com/papamarkou/eeyore}.
Source code for the examples of Section \ref{experiments}
can be found in \texttt{dmcl\_examples},
forming a separate \texttt{Python} package
based on \texttt{eeyore}.
\texttt{dmcl\_examples} can be downloaded from
\url{https://github.com/papamarkou/dmcl\_examples}.

\section*{Appendix A: blocked Gibbs}
\sname{Appendix A}
\label{app:bg}

Algorithm~\ref{bg} summarizes blocked Gibbs sampling
in the context of sampling MLP parameters,
as set out in Subsection~\ref{bgs}.
\tp{
	Remark~\ref{mwbg_ac_prop} and
	Remark~\ref{mwbg_ac_corol}
	provide expressions for the acceptance probability
	of candidate state $\theta_{z(q)}^{\star}$
	of Algorithm~\ref{bmwg_e} (MWBG sampling),
	as stated in Equation~\eqref{mwbg_ac_e_eq}
	of Subsection~\ref{ss:bgs_via_cs}.
}

\begin{algorithm*}[t] 
	\caption{Blocked Gibbs sampling}
	\label{bg}
	\begin{algorithmic}[1]
		\State{\textbf{Input}:
			training dataset $\mathcal{D}_{1:s}$
		}
		\State{\textbf{Input}:
			initial state
			$ \theta_{z(1):z(m)}^{(0)}$
		}
		\State{\textbf{Input}:
			number of sampling iterations $v$
		}
		
		$~$
		
		\For{$t=1,\ldots,v$}
		\State{Draw
			$\theta_{z(1)}^{(t)}\sim
			p(\theta_{z(1)} \vert
			\theta_{z(2):z(m)}^{(t-1)},
			\mathcal{D}_{1:s})$
		}
		\State{Draw
			$\theta_{z(2)}^{(t)}\sim
			p(\theta_{z(2)} \vert
			\theta_{z(1)}^{(t)},
			\theta_{z(3):z(m)}^{(t-1)},
			\mathcal{D}_{1:s})$
		}
		
		$\vdots$
		
		\State{Draw
			$\theta_{z(q)}^{(t)}\sim
			p(\theta_{z(q)} \vert
			\theta_{z(1):z(q-1)}^{(t)},
			\theta_{z(q+1):z(m)}^{(t-1)},
			\mathcal{D}_{1:s})$
		}
		
		$\vdots$
		
		\State{Draw
			$\theta_{z(m)}^{(t)}\sim
			p(\theta_{z(m)} \vert
			\theta_{z(1):z(m-1)}^{(t)},
			\mathcal{D}_{1:s})$
		}		
		\EndFor
	\end{algorithmic}
\end{algorithm*}


\tp{
\begin{prop}
	\label{mwbg_ac_prop}
	Consider an $\mbox{MLP}(\kappa_{0:\rho})$
	with likelihood function
	$\mathcal{L}(y_{1:s} \vert x_{1:s}, \theta)$
	specified by Equation~\eqref{mc_mlp_lik},
	where
	$\{(x_{i}, y_{i}):~i=1,2,\dots,s\}$
	is a training dataset related to
	a supervised classification problem
	and $\theta$ are the MLP parameters.
	Let $\pi (\theta) = \prod_{q=1}^{m}\pi (\theta_{z(q)})$
	be a parameter prior density based on a partition
	$\{\theta_{z(1)},\theta_{z(2)},\ldots,\theta_{z(m)}\}$
	of $\theta$.
	A MWBG version of Algorithm~\ref{bg}
	is used for sampling from the target density
	$p(\theta \vert x_{1:s}, y_{1:s})$.
	At iteration $t$,
	a candidate state $\theta_{z(q)}^{\star}$
	for parameter block $\theta_{z(q)}$
	is drawn from the isotropic normal proposal density
	$\mathcal{N}(\theta_{z(q)}^{(t-1)}, \sigma_q^2 I_q)$.
	The acceptance probability
	$a (\theta_{z(q)}^{\star}, \theta_{z(q)}^{(t-1)})$
	of $\theta_{z(q)}^{\star}$ is given by
	\begin{equation}
	\label{mwbg_ac_l_eq}
	\begin{split}
	& a (\theta_{z(q)}^{\star}, \theta_{z(q)}^{(t-1)}) = \\
	& \min{\Biggl\{
		\displaystyle
		\frac{\mathcal{L}(y_{1:s} \vert x_{1:s}, \theta^{\star})
			\pi (\theta_{z(q)}^{\star})}
		{\mathcal{L}(y_{1:s} \vert x_{1:s}, \theta^{(t-1)})
			\pi (\theta_{z(q)}^{(t-1)})},
		1\Biggr\} },
	\end{split}
	\end{equation}
	where
	$\theta^{(t-1)}$ and $\theta^{\star}$
	denote the values of $\theta$
	obtained by inverting the permutations
	$
	(
	\theta_{z(1):z(q-1)}^{(t)},
	\theta_{z(q):z(m)}^{(t-1)}
	)$
	and
	$
	(
	\theta_{z(1):z(q-1)}^{(t)},
	\theta_{z(q)}^{\star},
	\theta_{z(q+1):z(m)}^{(t-1)}
	)$,
	respectively.
\end{prop}
}

\tp{
\begin{prop}
	\label{mwbg_ac_corol}
	Consider an $\mbox{MLP}(\kappa_{0:\rho})$ with
	cross-entropy loss function
	$\mathcal{E}(\theta,\mathcal{D}_{1:s})$,
	where
	$\mathcal{D}_{1:s}=\{(x_{i}, y_{i}):~i=1,2,\dots,s\}$
	is a training dataset
	related to
	a supervised classification problem
	and $\theta$ are the MLP parameters.
	It is assumed that $\mathcal{E}$ is unnormalized,
	which means that it is not scaled
	by batch size.
	Under the sampling setup of Remark~\ref{mwbg_ac_prop},
	the acceptance probability of $\theta_{z(q)}^{\star}$,
	expressed in terms of cross-entropy loss function $\mathcal{E}$,
	is given by
	\begin{equation}
	\label{mwbg_ac_e_eq_app}
\begin{split}
& a (\theta_{z(q)}^{\star}, \theta_{z(q)}^{(t-1)}) = \\
&
\min{\left\{
	\displaystyle
	\frac{\pi (\theta_{z(q)}^{\star})  \exp{\left(
			\mathcal{E}(\theta^{(t-1)}, \mathcal{D}_{1:s})
			\right)} }
	{\pi (\theta_{z(q)}^{(t-1)})
		\exp{\left(
			\mathcal{E}(\theta^{\star}, \mathcal{D}_{1:s})
			\right)} },
	1\right\}}.
\end{split}
	\end{equation}
\end{prop}
}


\tp{
The relation between the cross-entropy loss function
	\begin{equation}
\label{app_a_eq_4}
\begin{split}
& \mathcal{E}(\theta, \mathcal{D}_{1:s}) =\\ 
& -
\sum_{i=1}^{s}
\sum_{k=1}^{\kappa_{\rho}}
\mathbbm{1}_{\{y_{i}=k\}}
\log{(h_{\rho,k}(x_{i},\theta))}
\end{split}
\end{equation}
and the likelihood function of Equation~\eqref{mc_mlp_lik}
is given by
	\begin{equation}
	\label{app_a_eq_5}
	\mathcal{L}(y_{1:s} \vert x_{1:s}, \theta ) =
	\exp{(-\mathcal{E}(\theta, \mathcal{D}_{1:s}))}.
	\end{equation}
Combining Equations~\eqref{mwbg_ac_l_eq} and~\eqref{app_a_eq_5}
yields Equation~\eqref{mwbg_ac_e_eq_app}. 
}

\tp{
Remark~\ref{mwbg_ac_prop}
states the acceptance probability
in statistical terms
using the likelihood function,
whereas Remark~\ref{mwbg_ac_corol}
states it in deep learning terms
using the cross-entropy loss function.
Remark~\ref{mwbg_ac_corol} is practical
in the sense that
deep learning software frameworks,
being geared towards optimization,
provide implementations of cross-entropy loss.
For example,
the unnormalized cross-entropy loss $\mathcal{E}$,
as stated in Equation~\eqref{app_a_eq_4},
can be computed in \texttt{PyTorch}
via the \texttt{CrossEntropyLoss} class
initialized with
\texttt{reduction=\textquotesingle sum\textquotesingle}.
}

\section*{Appendix B: noisy XOR}
\sname{Appendix B}
\label{app:noisy_xor}

Figure~\ref{noisy_xor_scatter} shows
the noisy XOR training and test datasets
used in Section~\ref{experiments}.
Information about how these noisy XOR datasets
have been simulated is available in~\cite{papamarkou2022}.

\begin{figure}[t]
	\centering
	\includegraphics[width=1\linewidth]{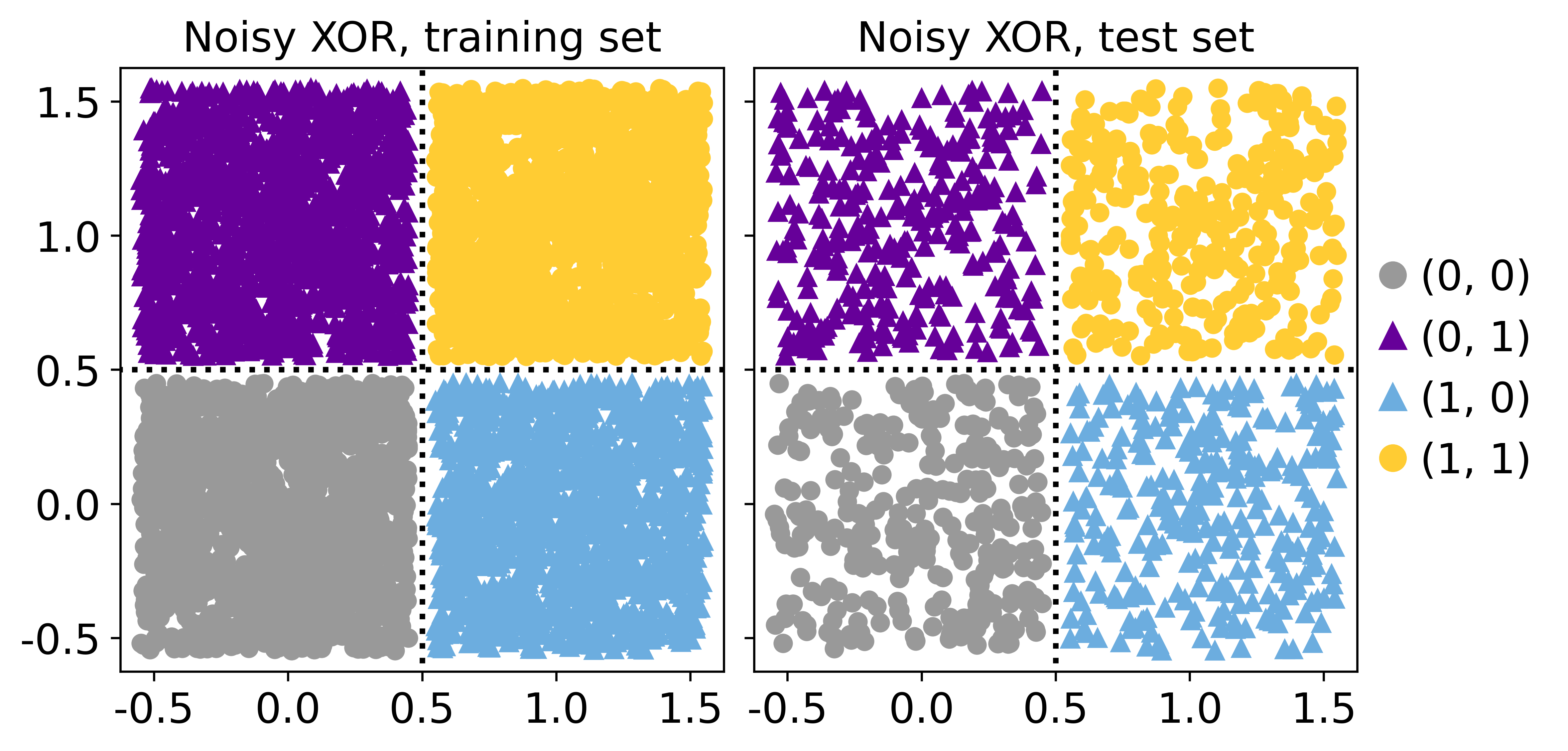}
	\caption{Noisy XOR training set (left) and test set (right)
		consisting of $5000$ and $1200$ data points, respectively.}
	\label{noisy_xor_scatter}
\end{figure}


\section*{Appendix C: tuning}
\sname{Appendix C}
\label{app:tuning}


Tables~\ref{mnist_tuning_layers}
and~\ref{fmnist_tuning_layers}
show that acceptance rates
obtained from minibatch FNBG sampling
can be retained at non-vanishing levels
in deeper layers
by reducing the proposal variances
corresponding to these layers.
$\mbox{MLP}(784, 10, 10, 10, 10)$
is fitted to MNIST and to FMNIST
via minibatch FNBG sampling
with different batch sizes.
The acceptance rate per layer
is computed from one chain
for each batch size.
Tables~\ref{mnist_tuning_layers}
and~\ref{fmnist_tuning_layers}
report the obtained acceptance rates
for MNIST and for FMNIST, respectively.

\begin{table}[t]
	\centering
	\caption{Acceptance rate per layer
		obtained by fitting
		$\mbox{MLP}(784, 10, 10, 10, 10)$
		to MNIST via minibatch FNBG sampling
		with different batch sizes.
	    $\sigma$ denotes the proposal standard deviation.}
	\label{mnist_tuning_layers}
	\begin{tabular}{clrr}
		\hline
		\multicolumn{2}{c}{Layer} & \multicolumn{1}{c}{$\sigma$} & Rate \\ \hline
		\multicolumn{4}{c}{$\mbox{Batch size}=600~(1\%)$} \\ \hline
		\multirow{3}{*}{\rotatebox[origin=c]{90}{Hidden}}
		& $1^{st}$                   & $5\cdot 10^{-2}$ & 45.56 \\
		& $2^{nd}$                   & $5\cdot 10^{-4}$ & 26.43 \\
		& $3^{rd}$                   & $5\cdot 10^{-4}$ & 26.28 \\ \cline{1-2}
		\multicolumn{2}{c}{Output} & $5\cdot 10^{-5}$ & 29.18 \\ \hline
		\multicolumn{4}{c}{$\mbox{Batch size}=1800~(3\%)$} \\ \hline
		\multirow{3}{*}{\rotatebox[origin=c]{90}{Hidden}}
		& $1^{st}$                   & $2\cdot 10^{-2}$ & 41.41 \\
		& $2^{nd}$                   & $2\cdot 10^{-4}$ & 30.68 \\
		& $3^{rd}$                   & $2\cdot 10^{-4}$ & 31.92 \\ \cline{1-2}
		\multicolumn{2}{c}{Output} & $2\cdot 10^{-5}$ & 35.66 \\ \hline
		\multicolumn{4}{c}{$\mbox{Batch size}=3000~(5\%)$} \\ \hline
		\multirow{3}{*}{\rotatebox[origin=c]{90}{Hidden}}
		& $1^{st}$                   &        $10^{-2}$ & 54.95 \\
		& $2^{nd}$                   &        $10^{-4}$ & 45.73 \\
		& $3^{rd}$                   &        $10^{-4}$ & 44.98 \\ \cline{1-2}
		\multicolumn{2}{c}{Output} &        $10^{-5}$ & 51.54 \\ \hline
		\multicolumn{4}{c}{$\mbox{Batch size}=4200~(7\%)$} \\ \hline
		\multirow{3}{*}{\rotatebox[origin=c]{90}{Hidden}}
		& $1^{st}$                   &        $10^{-2}$ & 31.68 \\
		& $2^{nd}$                   &        $10^{-4}$ & 20.17 \\
		& $3^{rd}$                   &        $10^{-4}$ & 19.76 \\ \cline{1-2}
		\multicolumn{2}{c}{Output} &        $10^{-5}$ & 22.22 \\ \hline
	\end{tabular}
\end{table}


\begin{table}[t]
	\centering
	\caption{Acceptance rate per layer
		obtained by fitting
		$\mbox{MLP}(784, 10, 10, 10, 10)$
		to FMNIST via minibatch FNBG sampling
		with different batch sizes.
	    $\sigma$ denotes the proposal standard deviation.}
	\label{fmnist_tuning_layers}
	\begin{tabular}{clrr}
		\hline
		\multicolumn{2}{c}{Layer} & \multicolumn{1}{c}{$\sigma$} & Rate \\ \hline
		\multicolumn{4}{c}{$\mbox{Batch size}=600~(1\%)$} \\ \hline
		\multirow{3}{*}{\rotatebox[origin=c]{90}{Hidden}}
		& $1^{st}$                   & $5\cdot 10^{-2}$ & 47.86 \\
		& $2^{nd}$                   & $5\cdot 10^{-4}$ & 34.61 \\
		& $3^{rd}$                   & $5\cdot 10^{-4}$ & 32.99 \\ \cline{1-2}
		\multicolumn{2}{c}{Output} & $5\cdot 10^{-5}$ & 37.73 \\ \hline
		\multicolumn{4}{c}{$\mbox{Batch size}=1800~(3\%)$} \\ \hline
		\multirow{3}{*}{\rotatebox[origin=c]{90}{Hidden}}
		& $1^{st}$                   & $2\cdot 10^{-2}$ & 60.34 \\
		& $2^{nd}$                   & $2\cdot 10^{-4}$ & 46.78 \\
		& $3^{rd}$                   & $2\cdot 10^{-4}$ & 45.91 \\ \cline{1-2}
		\multicolumn{2}{c}{Output} & $2\cdot 10^{-5}$ & 52.07 \\ \hline
		\multicolumn{4}{c}{$\mbox{Batch size}=3000~(5\%)$} \\ \hline
		\multirow{3}{*}{\rotatebox[origin=c]{90}{Hidden}}
		& $1^{st}$                   & $10^{-2}$ & 66.94 \\
		& $2^{nd}$                   & $10^{-4}$ & 57.40 \\
		& $3^{rd}$                   & $10^{-4}$ & 58.48 \\ \cline{1-2}
		\multicolumn{2}{c}{Output} & $10^{-5}$ & 64.64 \\ \hline
		\multicolumn{4}{c}{$\mbox{Batch size}=4200~(7\%)$} \\ \hline
		\multirow{3}{*}{\rotatebox[origin=c]{90}{Hidden}}
		& $1^{st}$                   & $10^{-2}$ & 55.28 \\
		& $2^{nd}$                   & $10^{-4}$ & 47.10 \\
		& $3^{rd}$                   & $10^{-4}$ & 47.19 \\ \cline{1-2}
		\multicolumn{2}{c}{Output} & $10^{-5}$ & 52.75 \\ \hline
	\end{tabular}
\end{table}


\section*{Appendix D: augmentation}
\sname{Appendix D}
\label{app:augmentation}

Figure~\ref{data_augmentation} shows
examples of images from the MNIST and FMNIST training sets
transformed by rotation, blurring and colour inversion.
These transformations are used
in Subsection~\ref{subs:data_aug_pred}
to assess the effect of data augmentation
on predictive accuracy.
Details about the performed transformations are available
in Subsection~\ref{subs:data_aug_pred}.

\begin{figure}[t]
	\begin{subfigure}{.491\textwidth}
		\centering
		\includegraphics[width=1\linewidth]{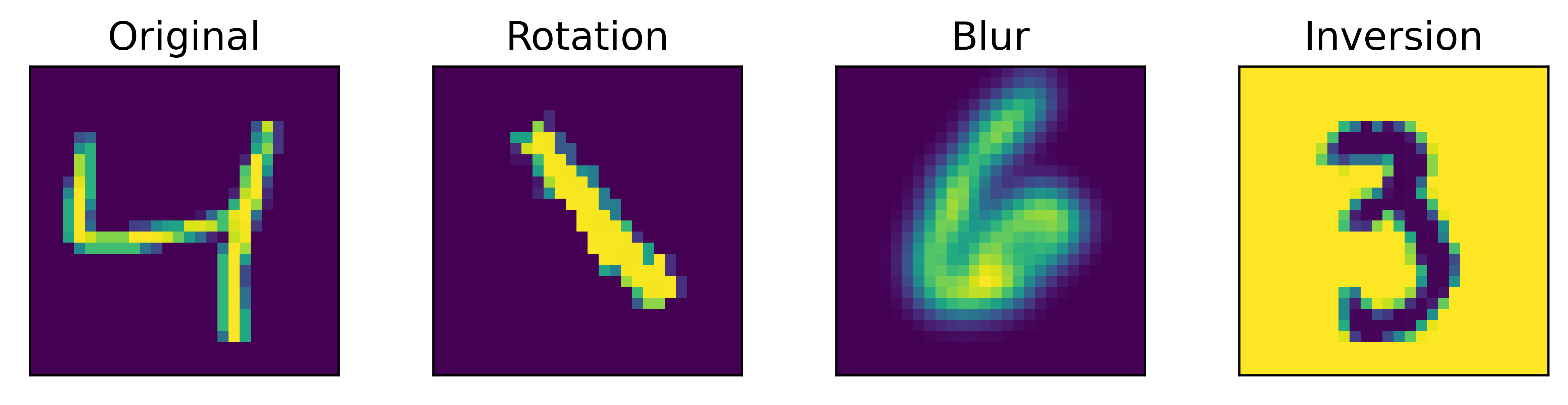}
		\caption{Examples of transformed MNIST training images.}
		\label{mnist_augmentation}
	\end{subfigure}\\
	\begin{subfigure}{.491\textwidth}
		\centering
		\includegraphics[width=1\linewidth]{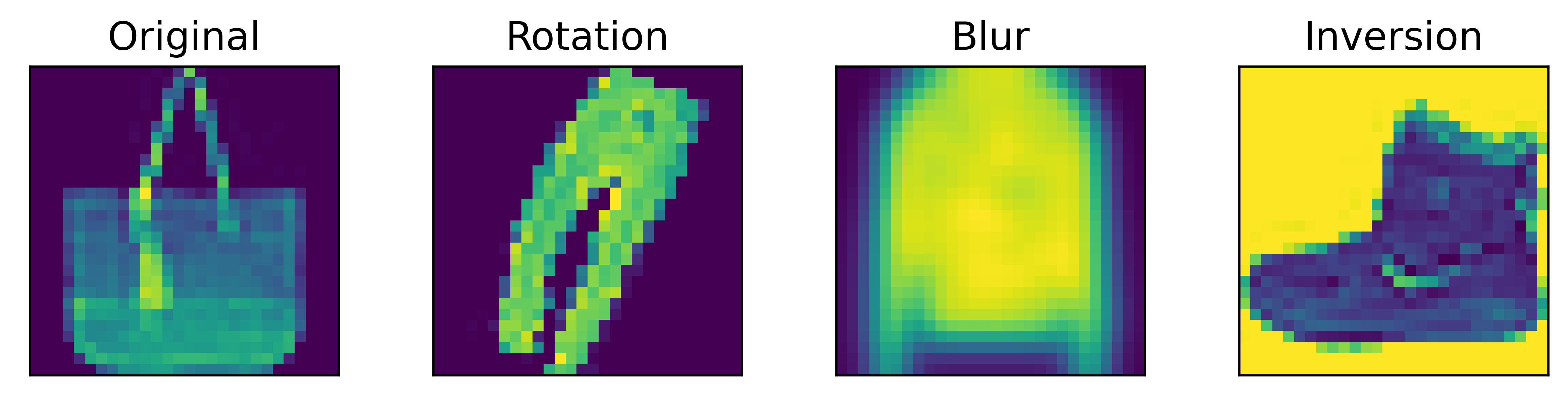}
		\caption{Examples of transformed FMNIST training images.}
		\label{fmnist_augmentation}
	\end{subfigure}
	\caption{Examples of MNIST and of FMNIST training images
		transformed by rotation, blurring and colour inversion.
		Such transformations are deployed
		in the data augmentation experiments
		of Subsection~\ref{subs:data_aug_pred}.
		Examples of untransformed MNIST and FMNIST training images
		are also displayed.}
	\label{data_augmentation}
\end{figure}

\section*{Acknowledgements}


The author would like to acknowledge the assistance given by Research IT
and the use of the Computational Shared Facility at The University of Manchester.
This work used the Cirrus UK National Tier-2 HPC Service at EPCC
(\url{http://www.cirrus.ac.uk})
funded by the University of Edinburgh and EPSRC (EP/P020267/1).
The author would like to thank Google
for the provision of free credit on Google Cloud Platform.

This work was presented
at two seminars supported by
a travel grant from the Dame Kathleen Ollerenshaw Trust,
which is gratefully acknowledged.

The author would like to dedicate this paper to the memory of his mother,
who died as this paper was being developed.


\bibliographystyle{imsart-nameyear}

\begin{thebibliography}{34}

\bibitem[\protect\citeauthoryear{Alexos, Boyd and Mandt}{2022}]{alexos2022}
\begin{binproceedings}[author]
\bauthor{\bsnm{Alexos},~\bfnm{Antonios}\binits{A.}},
  \bauthor{\bsnm{Boyd},~\bfnm{Alex~J}\binits{A.~J.}} \AND
  \bauthor{\bsnm{Mandt},~\bfnm{Stephan}\binits{S.}}
(\byear{2022}).
\btitle{Structured stochastic gradient {MCMC}}.
In \bbooktitle{Proceedings of the 39th International Conference on Machine
  Learning}
\bvolume{162}
\bpages{414--434}.
\bpublisher{PMLR}, \baddress{Baltimore, MD, USA}.
\end{binproceedings}
\endbibitem

\bibitem[\protect\citeauthoryear{Andrieu, de~Freitas and
  Doucet}{1999}]{andrieu1999}
\begin{bmisc}[author]
\bauthor{\bsnm{Andrieu},~\bfnm{C.}\binits{C.}}, \bauthor{\bparticle{de}
  \bsnm{Freitas},~\bfnm{J.~F.~G.}\binits{J.~F.~G.}} \AND
  \bauthor{\bsnm{Doucet},~\bfnm{A.}\binits{A.}}
(\byear{1999}).
\btitle{Sequential {B}ayesian estimation and model selection applied to neural
  networks}.
\end{bmisc}
\endbibitem

\bibitem[\protect\citeauthoryear{Andrieu, de~Freitas and
  Doucet}{2000}]{andrieu2000}
\begin{binproceedings}[author]
\bauthor{\bsnm{Andrieu},~\bfnm{Christophe}\binits{C.}}, \bauthor{\bparticle{de}
  \bsnm{Freitas},~\bfnm{Nando}\binits{N.}} \AND
  \bauthor{\bsnm{Doucet},~\bfnm{Arnaud}\binits{A.}}
(\byear{2000}).
\btitle{Reversible jump {MCMC} simulated annealing for neural networks}.
In \bbooktitle{Proceedings of the 16th Conference on Uncertainty in Artificial
  Intelligence}
\bpages{11--18}.
\end{binproceedings}
\endbibitem

\bibitem[\protect\citeauthoryear{Andrieu and Thoms}{2008}]{andrieu2008}
\begin{barticle}[author]
\bauthor{\bsnm{Andrieu},~\bfnm{Christophe}\binits{C.}} \AND
  \bauthor{\bsnm{Thoms},~\bfnm{Johannes}\binits{J.}}
(\byear{2008}).
\btitle{A tutorial on adaptive {MCMC}}.
\bjournal{Statistics and Computing}
\bvolume{18}
\bpages{343--373}.
\end{barticle}
\endbibitem

\bibitem[\protect\citeauthoryear{Bardenet, Doucet and
  Holmes}{2014}]{bardenet2014}
\begin{binproceedings}[author]
\bauthor{\bsnm{Bardenet},~\bfnm{Rémi}\binits{R.}},
  \bauthor{\bsnm{Doucet},~\bfnm{Arnaud}\binits{A.}} \AND
  \bauthor{\bsnm{Holmes},~\bfnm{Chris}\binits{C.}}
(\byear{2014}).
\btitle{Towards scaling up {M}arkov chain {M}onte {C}arlo: an adaptive
  subsampling approach}.
In \bbooktitle{Proceedings of the 31st International Conference on Machine
  Learning}.
\bseries{PMLR}
\bvolume{32}
\bpages{405--413}.
\end{binproceedings}
\endbibitem

\bibitem[\protect\citeauthoryear{Bouchard-C{{\^o}}t{{\'e}}, Doucet and
  Roth}{2017}]{bouchard2017}
\begin{barticle}[author]
\bauthor{\bsnm{Bouchard-C{{\^o}}t{{\'e}}},~\bfnm{Alexandre}\binits{A.}},
  \bauthor{\bsnm{Doucet},~\bfnm{Arnaud}\binits{A.}} \AND
  \bauthor{\bsnm{Roth},~\bfnm{Andrew}\binits{A.}}
(\byear{2017}).
\btitle{Particle {G}ibbs split-merge sampling for {B}ayesian inference in
  mixture models}.
\bjournal{Journal of Machine Learning Research}
\bvolume{18}
\bpages{1--39}.
\end{barticle}
\endbibitem

\bibitem[\protect\citeauthoryear{Chen, Fox and Guestrin}{2014}]{chen2014}
\begin{binproceedings}[author]
\bauthor{\bsnm{Chen},~\bfnm{Tianqi}\binits{T.}},
  \bauthor{\bsnm{Fox},~\bfnm{Emily}\binits{E.}} \AND
  \bauthor{\bsnm{Guestrin},~\bfnm{Carlos}\binits{C.}}
(\byear{2014}).
\btitle{Stochastic gradient {H}amiltonian {M}onte {C}arlo}.
In \bbooktitle{Proceedings of the 31st International Conference on Machine
  Learning}.
\bseries{PMLR}
\bvolume{32}
\bpages{1683--1691}.
\end{binproceedings}
\endbibitem

\bibitem[\protect\citeauthoryear{de~Freitas}{1999}]{freitas1999}
\begin{bphdthesis}[author]
\bauthor{\bparticle{de} \bsnm{Freitas},~\bfnm{Nando}\binits{N.}}
(\byear{1999}).
\btitle{Bayesian methods for neural networks},
\btype{PhD thesis},
\bpublisher{University of Cambridge}.
\end{bphdthesis}
\endbibitem

\bibitem[\protect\citeauthoryear{de~Freitas et~al.}{2001}]{freitas2001}
\begin{binbook}[author]
\bauthor{\bparticle{de} \bsnm{Freitas},~\bfnm{N.}\binits{N.}},
  \bauthor{\bsnm{Andrieu},~\bfnm{C.}\binits{C.}},
  \bauthor{\bsnm{H{\o}jen-S{\o}rensen},~\bfnm{P.}\binits{P.}},
  \bauthor{\bsnm{Niranjan},~\bfnm{M.}\binits{M.}} \AND
  \bauthor{\bsnm{Gee},~\bfnm{A.}\binits{A.}}
(\byear{2001}).
\btitle{Sequential {M}onte {C}arlo methods for neural networks}
In \bbooktitle{Sequential Monte Carlo Methods in Practice}
\bpages{359--379}.
\bpublisher{Springer}, \baddress{New York, NY, USA}.
\end{binbook}
\endbibitem

\bibitem[\protect\citeauthoryear{Girolami and Calderhead}{2011}]{girolami2011}
\begin{barticle}[author]
\bauthor{\bsnm{Girolami},~\bfnm{Mark}\binits{M.}} \AND
  \bauthor{\bsnm{Calderhead},~\bfnm{Ben}\binits{B.}}
(\byear{2011}).
\btitle{Riemann manifold {L}angevin and {H}amiltonian {M}onte {C}arlo methods}.
\bjournal{Journal of the Royal Statistical Society: Series B (Statistical
  Methodology)}
\bvolume{73}
\bpages{123--214}.
\end{barticle}
\endbibitem

\bibitem[\protect\citeauthoryear{Gong, Li and
  Hernández-Lobato}{2019}]{gong2018}
\begin{binproceedings}[author]
\bauthor{\bsnm{Gong},~\bfnm{Wenbo}\binits{W.}},
  \bauthor{\bsnm{Li},~\bfnm{Yingzhen}\binits{Y.}} \AND
  \bauthor{\bsnm{Hernández-Lobato},~\bfnm{José~Miguel}\binits{J.~M.}}
(\byear{2019}).
\btitle{Meta-learning for stochastic gradient {MCMC}}.
In \bbooktitle{International Conference on Learning Representations}.
\bseries{PMLR}.
\end{binproceedings}
\endbibitem

\bibitem[\protect\citeauthoryear{Grathwohl et~al.}{2021}]{grathwohl2021}
\begin{binproceedings}[author]
\bauthor{\bsnm{Grathwohl},~\bfnm{Will}\binits{W.}},
  \bauthor{\bsnm{Swersky},~\bfnm{Kevin}\binits{K.}},
  \bauthor{\bsnm{Hashemi},~\bfnm{Milad}\binits{M.}},
  \bauthor{\bsnm{Duvenaud},~\bfnm{David}\binits{D.}} \AND
  \bauthor{\bsnm{Maddison},~\bfnm{Chris}\binits{C.}}
(\byear{2021}).
\btitle{Oops I took A gradient: scalable sampling for discrete distributions}.
In \bbooktitle{Proceedings of the 38th International Conference on Machine
  Learning}.
\bseries{PMLR}
\bvolume{139}
\bpages{3831--3841}.
\end{binproceedings}
\endbibitem

\bibitem[\protect\citeauthoryear{Hastie, Tibshirani and
  Friedman}{2016}]{hastie2016}
\begin{bbook}[author]
\bauthor{\bsnm{Hastie},~\bfnm{Trevor}\binits{T.}},
  \bauthor{\bsnm{Tibshirani},~\bfnm{Robert}\binits{R.}} \AND
  \bauthor{\bsnm{Friedman},~\bfnm{Jerome}\binits{J.}}
(\byear{2016}).
\btitle{The elements of statistical learning: data mining, inference and
  prediction},
\bedition{second} ed.
\bpublisher{Springer}, \baddress{New York, NY, USA}.
\end{bbook}
\endbibitem

\bibitem[\protect\citeauthoryear{He et~al.}{2016}]{he2016}
\begin{binproceedings}[author]
\bauthor{\bsnm{He},~\bfnm{Kaiming}\binits{K.}},
  \bauthor{\bsnm{Zhang},~\bfnm{Xiangyu}\binits{X.}},
  \bauthor{\bsnm{Ren},~\bfnm{Shaoqing}\binits{S.}} \AND
  \bauthor{\bsnm{Sun},~\bfnm{Jian}\binits{J.}}
(\byear{2016}).
\btitle{Deep residual learning for image recognition}.
In \bbooktitle{2016 IEEE Conference on Computer Vision and Pattern Recognition}
\bpages{770--778}.
\end{binproceedings}
\endbibitem

\bibitem[\protect\citeauthoryear{Izmailov et~al.}{2021}]{izmailov2021}
\begin{binproceedings}[author]
\bauthor{\bsnm{Izmailov},~\bfnm{Pavel}\binits{P.}},
  \bauthor{\bsnm{Vikram},~\bfnm{Sharad}\binits{S.}},
  \bauthor{\bsnm{Hoffman},~\bfnm{Matthew~D}\binits{M.~D.}} \AND
  \bauthor{\bsnm{Wilson},~\bfnm{Andrew Gordon~Gordon}\binits{A.~G.~G.}}
(\byear{2021}).
\btitle{What are {B}ayesian neural network posteriors really like?}
In \bbooktitle{Proceedings of the 38th International Conference on Machine
  Learning}
\bvolume{139}
\bpages{4629--4640}.
\bpublisher{PMLR}, \baddress{virtual}.
\end{binproceedings}
\endbibitem

\bibitem[\protect\citeauthoryear{Krizhevsky and Hinton}{2009}]{krizhevsky2009}
\begin{btechreport}[author]
\bauthor{\bsnm{Krizhevsky},~\bfnm{Alex}\binits{A.}} \AND
  \bauthor{\bsnm{Hinton},~\bfnm{Geoffrey}\binits{G.}}
(\byear{2009}).
\btitle{Learning multiple layers of features from tiny images}
\btype{Technical Report},
\bpublisher{University of Toronto},
\baddress{Toronto, Ontario}.
\end{btechreport}
\endbibitem

\bibitem[\protect\citeauthoryear{{\L}atuszy{\'{n}}ski, Roberts and
  Rosenthal}{2013}]{latuszynski2013}
\begin{barticle}[author]
\bauthor{\bsnm{{\L}atuszy{\'{n}}ski},~\bfnm{Krzysztof}\binits{K.}},
  \bauthor{\bsnm{Roberts},~\bfnm{Gareth~O.}\binits{G.~O.}} \AND
  \bauthor{\bsnm{Rosenthal},~\bfnm{Jeffrey~S.}\binits{J.~S.}}
(\byear{2013}).
\btitle{Adaptive {G}ibbs samplers and related {MCMC} methods}.
\bjournal{The Annals of Applied Probability}
\bvolume{23}
\bpages{66--98}.
\end{barticle}
\endbibitem

\bibitem[\protect\citeauthoryear{Lecun et~al.}{1998}]{lecun1998}
\begin{barticle}[author]
\bauthor{\bsnm{Lecun},~\bfnm{Y.}\binits{Y.}},
  \bauthor{\bsnm{Bottou},~\bfnm{L.}\binits{L.}},
  \bauthor{\bsnm{Bengio},~\bfnm{Y.}\binits{Y.}} \AND
  \bauthor{\bsnm{Haffner},~\bfnm{P.}\binits{P.}}
(\byear{1998}).
\btitle{Gradient-based learning applied to document recognition}.
\bjournal{Proceedings of the IEEE}
\bvolume{86}
\bpages{2278--2324}.
\end{barticle}
\endbibitem

\bibitem[\protect\citeauthoryear{Matsubara, Oates and
  Briol}{2021}]{matsubara2021}
\begin{barticle}[author]
\bauthor{\bsnm{Matsubara},~\bfnm{Takuo}\binits{T.}},
  \bauthor{\bsnm{Oates},~\bfnm{Chris~J.}\binits{C.~J.}} \AND
  \bauthor{\bsnm{Briol},~\bfnm{Fran{\c{c}}ois-Xavier}\binits{F.-X.}}
(\byear{2021}).
\btitle{The ridgelet prior: a covariance function approach to prior
  specification for {B}ayesian neural networks}.
\bjournal{Journal of Machine Learning Research}
\bvolume{22}
\bpages{1--57}.
\end{barticle}
\endbibitem

\bibitem[\protect\citeauthoryear{Minsky and Papert}{1988}]{minsky1988}
\begin{bbook}[author]
\bauthor{\bsnm{Minsky},~\bfnm{Marvin~L}\binits{M.~L.}} \AND
  \bauthor{\bsnm{Papert},~\bfnm{Seymour~A}\binits{S.~A.}}
(\byear{1988}).
\btitle{Perceptrons: expanded edition}.
\bpublisher{MIT press}, \baddress{Cambridge, MA, USA}.
\end{bbook}
\endbibitem

\bibitem[\protect\citeauthoryear{Netzer et~al.}{2011}]{netzer2011}
\begin{binproceedings}[author]
\bauthor{\bsnm{Netzer},~\bfnm{Yuval}\binits{Y.}},
  \bauthor{\bsnm{Wang},~\bfnm{Tao}\binits{T.}},
  \bauthor{\bsnm{Coates},~\bfnm{Adam}\binits{A.}},
  \bauthor{\bsnm{Bissacco},~\bfnm{Alessandro}\binits{A.}},
  \bauthor{\bsnm{Wu},~\bfnm{Bo}\binits{B.}} \AND
  \bauthor{\bsnm{Ng},~\bfnm{Andrew~Y.}\binits{A.~Y.}}
(\byear{2011}).
\btitle{Reading digits in natural images with unsupervised feature learning}.
In \bbooktitle{NIPS Workshop on Deep Learning and Unsupervised Feature
  Learning}.
\end{binproceedings}
\endbibitem

\bibitem[\protect\citeauthoryear{Papamarkou et~al.}{2022}]{papamarkou2022}
\begin{barticle}[author]
\bauthor{\bsnm{Papamarkou},~\bfnm{T.}\binits{T.}},
  \bauthor{\bsnm{Hinkle},~\bfnm{J.}\binits{J.}},
  \bauthor{\bsnm{Young},~\bfnm{M.~T.}\binits{M.~T.}} \AND
  \bauthor{\bsnm{Womble},~\bfnm{D.}\binits{D.}}
(\byear{2022}).
\btitle{Challenges in {M}arkov chain {M}onte {C}arlo for {B}ayesian neural
  networks}.
\bjournal{Statistical Science}
\bvolume{37}
\bpages{425--442}.
\end{barticle}
\endbibitem

\bibitem[\protect\citeauthoryear{Roberts and Sahu}{1997}]{roberts1997}
\begin{barticle}[author]
\bauthor{\bsnm{Roberts},~\bfnm{G.~O.}\binits{G.~O.}} \AND
  \bauthor{\bsnm{Sahu},~\bfnm{S.~K.}\binits{S.~K.}}
(\byear{1997}).
\btitle{Updating schemes, correlation structure, blocking and parameterization
  for the {G}ibbs sampler}.
\bjournal{Journal of the Royal Statistical Society: Series B (Statistical
  Methodology)}
\bvolume{59}
\bpages{291--317}.
\end{barticle}
\endbibitem

\bibitem[\protect\citeauthoryear{Rosenblatt}{1958}]{rosenblatt1958}
\begin{barticle}[author]
\bauthor{\bsnm{Rosenblatt},~\bfnm{Frank}\binits{F.}}
(\byear{1958}).
\btitle{The perceptron: a probabilistic model for information storage and
  organization in the brain}.
\bjournal{Psychological review}
\bvolume{65}
\bpages{386}.
\end{barticle}
\endbibitem

\bibitem[\protect\citeauthoryear{Saul and Jordan}{1995}]{saul1995}
\begin{binproceedings}[author]
\bauthor{\bsnm{Saul},~\bfnm{Lawrence}\binits{L.}} \AND
  \bauthor{\bsnm{Jordan},~\bfnm{Michael}\binits{M.}}
(\byear{1995}).
\btitle{Exploiting Tractable Substructures in Intractable Networks}.
In \bbooktitle{Advances in Neural Information Processing Systems}
\bvolume{8}.
\bpublisher{MIT Press}, \baddress{Denver, CO, USA}.
\end{binproceedings}
\endbibitem

\bibitem[\protect\citeauthoryear{Titterington}{2004}]{titterington2004}
\begin{barticle}[author]
\bauthor{\bsnm{Titterington},~\bfnm{D.~M.}\binits{D.~M.}}
(\byear{2004}).
\btitle{Bayesian methods for neural networks and related models}.
\bjournal{Statistical Science}
\bvolume{19}
\bpages{128--139}.
\end{barticle}
\endbibitem

\bibitem[\protect\citeauthoryear{Tran et~al.}{2022}]{tran2022}
\begin{barticle}[author]
\bauthor{\bsnm{Tran},~\bfnm{Ba-Hien}\binits{B.-H.}},
  \bauthor{\bsnm{Rossi},~\bfnm{Simone}\binits{S.}},
  \bauthor{\bsnm{Milios},~\bfnm{Dimitrios}\binits{D.}} \AND
  \bauthor{\bsnm{Filippone},~\bfnm{Maurizio}\binits{M.}}
(\byear{2022}).
\btitle{All you need is a good functional prior for {B}ayesian deep learning}.
\bjournal{Journal of Machine Learning Research}
\bvolume{23}
\bpages{1--56}.
\end{barticle}
\endbibitem

\bibitem[\protect\citeauthoryear{Vono, Dobigeon and Chainais}{2019}]{vono2019}
\begin{barticle}[author]
\bauthor{\bsnm{Vono},~\bfnm{Maxime}\binits{M.}},
  \bauthor{\bsnm{Dobigeon},~\bfnm{Nicolas}\binits{N.}} \AND
  \bauthor{\bsnm{Chainais},~\bfnm{Pierre}\binits{P.}}
(\byear{2019}).
\btitle{Split-and-augmented {G}ibbs sampler—application to large-scale
  inference problems}.
\bjournal{IEEE Transactions on Signal Processing}
\bvolume{67}
\bpages{1648--1661}.
\end{barticle}
\endbibitem

\bibitem[\protect\citeauthoryear{Vono, Paulin and Doucet}{2022}]{vono2022}
\begin{barticle}[author]
\bauthor{\bsnm{Vono},~\bfnm{Maxime}\binits{M.}},
  \bauthor{\bsnm{Paulin},~\bfnm{Daniel}\binits{D.}} \AND
  \bauthor{\bsnm{Doucet},~\bfnm{Arnaud}\binits{A.}}
(\byear{2022}).
\btitle{Efficient {MCMC} sampling with dimension-free convergence rate using
  {ADMM}-type splitting}.
\bjournal{Journal of Machine Learning Research}
\bvolume{23}
\bpages{1--69}.
\end{barticle}
\endbibitem

\bibitem[\protect\citeauthoryear{Welling and Teh}{2011}]{welling2011}
\begin{binproceedings}[author]
\bauthor{\bsnm{Welling},~\bfnm{Max}\binits{M.}} \AND
  \bauthor{\bsnm{Teh},~\bfnm{Yee~Whye}\binits{Y.~W.}}
(\byear{2011}).
\btitle{Bayesian learning via stochastic gradient {L}angevin dynamics}.
In \bbooktitle{Proceedings of the 28th International Conference on Machine
  Learning}
\bpages{681--688}.
\end{binproceedings}
\endbibitem

\bibitem[\protect\citeauthoryear{Wenzel et~al.}{2020}]{wenzel2020}
\begin{binproceedings}[author]
\bauthor{\bsnm{Wenzel},~\bfnm{Florian}\binits{F.}},
  \bauthor{\bsnm{Roth},~\bfnm{Kevin}\binits{K.}},
  \bauthor{\bsnm{Veeling},~\bfnm{Bastiaan}\binits{B.}},
  \bauthor{\bsnm{Swiatkowski},~\bfnm{Jakub}\binits{J.}},
  \bauthor{\bsnm{Tran},~\bfnm{Linh}\binits{L.}},
  \bauthor{\bsnm{Mandt},~\bfnm{Stephan}\binits{S.}},
  \bauthor{\bsnm{Snoek},~\bfnm{Jasper}\binits{J.}},
  \bauthor{\bsnm{Salimans},~\bfnm{Tim}\binits{T.}},
  \bauthor{\bsnm{Jenatton},~\bfnm{Rodolphe}\binits{R.}} \AND
  \bauthor{\bsnm{Nowozin},~\bfnm{Sebastian}\binits{S.}}
(\byear{2020}).
\btitle{How good is the {B}ayes posterior in deep neural networks really?}
In \bbooktitle{Proceedings of the 37th International Conference on Machine
  Learning}
\bvolume{119}
\bpages{10248--10259}.
\bpublisher{PMLR}, \baddress{Vienna, Austria}.
\end{binproceedings}
\endbibitem

\bibitem[\protect\citeauthoryear{Wiese et~al.}{2023}]{wiese2023}
\begin{binproceedings}[author]
\bauthor{\bsnm{Wiese},~\bfnm{Jonas~Gregor}\binits{J.~G.}},
  \bauthor{\bsnm{Wimmer},~\bfnm{Lisa}\binits{L.}},
  \bauthor{\bsnm{Papamarkou},~\bfnm{Theodore}\binits{T.}},
  \bauthor{\bsnm{Bischl},~\bfnm{Bernd}\binits{B.}},
  \bauthor{\bsnm{G\"unnemann},~\bfnm{Stephan}\binits{S.}} \AND
  \bauthor{\bsnm{R\"ugamer},~\bfnm{David}\binits{D.}}
(\byear{2023}).
\btitle{Towards efficient {MCMC} sampling in {B}ayesian neural networks by
  exploiting symmetry}.
In \bbooktitle{European Conference on Machine Learning and Principles and
  Practice of Knowledge Discovery in Databases}.
\bpublisher{Springer Nature}, \baddress{Turin, Italy}.
\end{binproceedings}
\endbibitem

\bibitem[\protect\citeauthoryear{Xiao, Rasul and Vollgraf}{2017}]{xiao2017}
\begin{barticle}[author]
\bauthor{\bsnm{Xiao},~\bfnm{Han}\binits{H.}},
  \bauthor{\bsnm{Rasul},~\bfnm{Kashif}\binits{K.}} \AND
  \bauthor{\bsnm{Vollgraf},~\bfnm{Roland}\binits{R.}}
(\byear{2017}).
\btitle{Fashion-{MNIST}: a novel image dataset for benchmarking machine
  learning algorithms}.
\bjournal{arXiv preprint arXiv:1708.07747}.
\end{barticle}
\endbibitem

\bibitem[\protect\citeauthoryear{Zhang et~al.}{2020}]{zhang2020}
\begin{binproceedings}[author]
\bauthor{\bsnm{Zhang},~\bfnm{Ruqi}\binits{R.}},
  \bauthor{\bsnm{Li},~\bfnm{Chunyuan}\binits{C.}},
  \bauthor{\bsnm{Zhang},~\bfnm{Jianyi}\binits{J.}},
  \bauthor{\bsnm{Chen},~\bfnm{Changyou}\binits{C.}} \AND
  \bauthor{\bsnm{Wilson},~\bfnm{Andrew~Gordon}\binits{A.~G.}}
(\byear{2020}).
\btitle{Cyclical stochastic gradient {MCMC} for {B}ayesian deep learning}.
In \bbooktitle{International Conference on Learning Representations}.
\end{binproceedings}
\endbibitem

\end{thebibliography}

\end{document}